\documentclass[twoside]{article}

\usepackage[preprint]{aistats2027}

\usepackage[round]{natbib}

\usepackage{amsmath,amssymb,amsthm}
\usepackage{bm}
\usepackage{booktabs}
\usepackage{graphicx}
\usepackage{algorithm}
\usepackage{algpseudocode}
\usepackage[colorlinks=true,linkcolor=blue,citecolor=blue,urlcolor=blue]{hyperref}

\newtheorem{proposition}{Proposition}
\newtheorem{lemma}{Lemma}
\newtheorem{theorem}{Theorem}

\newtheorem{corollary}{Corollary}

\newcommand{\R}{\mathbb{R}}
\newcommand{\E}{\mathbb{E}}
\newcommand{\N}{\mathcal{N}}
\newcommand{\diag}{\operatorname{diag}}
\newcommand{\Cov}{\operatorname{Cov}}
\newcommand{\Var}{\operatorname{Var}}
\newcommand{\vecop}{\operatorname{vec}}
\newcommand{\bpsi}{\bm{\psi}}

\begin{document}

\runningtitle{Scalable Logistic GP Density Regression with Kinetic Langevin Sampling}
\runningauthor{Daniel Paulin, \'Ad\'am Jung, Andr\'as A.~Bencz\'ur}

\twocolumn[
\aistatstitle{Scalable Logistic Gaussian Process Density Regression\\ with Kinetic Langevin Sampling}
\aistatsauthor{Daniel Paulin \And \'Ad\'am Jung \And Andr\'as A.~Bencz\'ur}
\aistatsaddress{Nanyang Technological University \And E\"otv\"os Lor\'and University (ELTE) \And HUN-REN Institute for Computer\\ Science and Control (SZTAKI)}
]

\begin{abstract}
Conditional density estimation targets the full distribution of a response given
covariates, as required, for example, for per-galaxy photometric redshifts. We develop a
scalable Bayesian estimator based on the logistic Gaussian process. The log conditional
density has a separable covariance: a Mat\'ern kernel along the response, represented in
a truncated Fourier basis on a circle, and a covariate kernel represented by Nystr\"om
features, which accommodate non-stationary kernels with input-dependent amplitudes and
length scales. Instead of a Laplace or variational approximation, we sample the latent
field of this finite-feature model. Given the
hyperparameters, its posterior is strongly log-concave with a uniformly bounded Hessian,
and we draw from it by simulating kinetic Langevin dynamics with symmetric minibatch
splitting in Kronecker-whitened coordinates. Marginal-likelihood gradients follow from Fisher's
identity as posterior expectations. Under the conditions of our analysis their bias is controlled by the
sampler's step size and run length, and the predictive averages over the
non-Gaussian latent posterior instead of a Gaussian around its mode. On photometric-redshift benchmarks with up to 3.9 million
training observations, trained on a single GPU, the estimator is competitive with
state-of-the-art tabular foundation models on density and calibration metrics.
\end{abstract}

\section{INTRODUCTION}

A regression model that returns one number per input answers only part of the
question. When the noise is heteroscedastic, skewed or multimodal, decisions depend on
the shape of the conditional distribution $p(y\mid x)$. Photometric redshift estimation
is the canonical example in astronomy: imaging surveys measure a handful of broad-band
magnitudes per galaxy, galaxies of different types at different redshifts produce
similar colours, and the redshift distribution given the photometry is often broad,
skewed or bimodal. Cosmological analyses stack these per-galaxy distributions into the
redshift distribution of a tomographic bin \citep{malz2022nz}, so the community scores
estimators on distributional criteria: the log-likelihood of the true redshift, the
uniformity of the probability integral transform (PIT), the conditional density
estimation (CDE) loss of \citet{izbicki2017cde} and the continuous ranked probability
score \citep{schmidt2020lsst}.

The logistic Gaussian process (LGP) is the classical Bayesian nonparametric model for
this problem \citep{leonard1978density,lenk1988logistic,tokdar2010lgp}. Its log-density
is a Gaussian process normalised over the response, which gives smooth, valid densities
with posterior uncertainty. \citet{riihimaki2014lgp} fit it by a Laplace approximation on a
grid, with Markov chain Monte Carlo (Metropolis--Hastings on the latent values, NUTS on the
hyperparameters) as the reference on small problems. \citet{chen2026gblgp} scale a
generalized-Bayes version, whose Hyv\"arinen-score posterior needs no normaliser, to more
than $150{,}000$ observations with inducing points and variational inference. Spatial logistic Gaussian processes (SLGPs) \citep{gautier2026continuous} model covariate-indexed density fields, with a random Fourier feature implementation supporting MAP, Laplace and NUTS inference \citep{gautier2026slgp}; we compare with it in Appendix~\ref{app:slgp}. Tabular foundation models go the other way.
TabPFN \citep{hollmann2025tabpfn} is pre-trained on synthetic tasks and predicts a full
density for every test input in one forward pass over the training set. In a recent benchmark it was
the best off-the-shelf conditional density estimator on most of 39 datasets, and it beat
every baseline in a photometric-redshift case study with ten times less data
\citep{izbicki2026benchmark}.

We show that a finite-feature LGP can be fitted by sampling its likelihood-based latent
posterior, without imposing a Gaussian approximation on it, at the scale of millions of
observations, and competitively with these models. Our contributions are:
\begin{itemize}
\setlength\itemsep{0pt}
\item a representation of the LGP with a separable covariance whose response axis is a
truncated Fourier series of a Mat\'ern process on a circle and whose covariate axis is a
finite feature map (Nystr\"om features for non-stationary kernels), under which every
observation enters the likelihood at its own response value, without binning
(Section~\ref{sec:model});
\item a proof that, given the hyperparameters, the latent posterior, strongly log-concave as
for SLGPs \citep{gautier2026slgp}, has a uniformly bounded Hessian and third derivative, with constants
independent of the number of features and of the number of modes (apart from a numerical
floor on the spectrum, Appendix~\ref{app:spectral}), for every input-dependent extension we
use, and a sampler, kinetic Langevin dynamics with symmetric minibatch splitting in
Kronecker-whitened coordinates, whose Wasserstein bias is then $O(h^2\sqrt d)$
(Section~\ref{sec:inference});
\item hyperparameter learning by Fisher-identity gradients from the warm-started chains,
whose bias is bounded under Theorem~\ref{thm:sms}'s assumptions for fixed representation and
hyperparameters, with a problem-dependent prefactor (Section~\ref{sec:fisher});
\item a non-stationary extension, a deep Gibbs kernel with input-dependent amplitude,
length scales and a per-object warp of the response (Section~\ref{sec:nonstat}), and an
evaluation on four photometric-redshift benchmarks with up to $3.9$ million training
galaxies (Section~\ref{sec:experiments}).
\end{itemize}
Code to reproduce all experiments is available at \url{https://github.com/paulindani/LGP}.

\section{THE LOGISTIC GAUSSIAN PROCESS MODEL}
\label{sec:model}

\paragraph{Density on a transformed response.}
A fixed, strictly increasing map $T$ onto $[0,1]$, the cumulative distribution of a
kernel density estimate of the training responses on a bounded interval
(Appendix~\ref{app:nonstat}), carries $y$ to $u = T(y)$; densities
are mapped back by $p(y\mid x) = p_u(T(y)\mid x)\,T'(y)$. The conditional density of $u$
is the exponential of a latent function,
\begin{equation}
p_u(u\mid x) = \frac{\exp\{g(u,x)\}}{Z(x)},\qquad Z(x) = \int_0^1 e^{g(u,x)}\,du,
\label{eq:lgp}
\end{equation}
with $g$ a zero-mean Gaussian process with separable covariance
$\Cov\{g(u,x),g(u',x')\} = C(u,u')\,K_x(x,x')$. The response kernel $C$ acts as a smoothness
prior on the shape of the density, and the covariate kernel $K_x$ lets that shape change
with $x$. Positivity and normalisation hold for every $g$.

\paragraph{Response axis: a Fourier basis on a circle.}
A stationary kernel on $[0,1]$ has no fast transform. We place the prior on a circle of
circumference $2$ of which $[0,1)$ is one half, the circulant-embedding construction
\citep{dietrich1997circulant}. The covariance on the circle is the periodised Mat\'ern
kernel $\sum_{n\in\mathbb Z}c(d+2n)$, $c(d)=\kappa_\nu(d/\ell_y)$, which is positive
definite for every length scale. By Poisson summation its Fourier coefficients are the
Mat\'ern spectral density $S_\nu(\omega;\ell_y)$ at $\omega = \pi k$. Truncated to $K$
modes,
\begin{equation}
g(u) = \sum_{k=1}^{K-1}A_k\bigl(a_k\cos\pi k u - b_k\sin\pi k u\bigr),
\label{eq:series}
\end{equation}
$A_k^2 = S_\nu(\pi k;\ell_y)$,
with $a_k, b_k\sim\N(0,1)$ and the constant mode omitted since it cancels in
\eqref{eq:lgp}. The weight of mode $K$ depends on $\pi K\ell_y$ alone, so $K$ is a
smoothness budget, independent of the sample size and of any grid (Appendix~\ref{app:spectral}).

\paragraph{Covariate axis: a feature map.}
With a finite feature map $\phi:\R^D\to\R^J$, $K_x(x,x')\approx\phi(x)^\top\phi(x')$,
the separable prior is reproduced exactly by a latent matrix $\xi\in\R^{J\times 2(K-1)}$ with
$\vecop(\xi)\sim\N(0,I)$:
\begin{equation}
g(u,x_n) = P_n^\top\bpsi(u),\qquad P_n = \phi(x_n)^\top\xi\in\R^{2(K-1)},
\label{eq:field}
\end{equation}
where $\bpsi(u) = (A_k\cos\pi ku,\,-A_k\sin\pi ku)_{1\le k<K}$, so the latent dimension is
$d = 2J(K-1)$. We use Nystr\"om features
\citep{williams2001nystrom}: with anchors $\bar x_1,\dots,\bar x_J$ (the $k$-means
centres of the training inputs) and a unit-variance kernel $k$,
\begin{equation}
\phi(x) = \sigma L^{-1}k_a(x),\qquad LL^\top = K_{aa} + \delta I,
\label{eq:nystrom}
\end{equation}
where $k_a(x) = (k(\bar x_j,x))_{j\le J}$ and $K_{aa} = (k(\bar x_i,\bar x_j))_{i,j\le J}$.
This is a regularised Nystr\"om approximation. On the anchors it equals
$\sigma^2K_{aa}(K_{aa}+\delta I)^{-1}K_{aa}$, whose eigenvalues $\sigma^2\lambda^2/(\lambda+\delta)$
differ from those of $\sigma^2K_{aa}$ by at most $\sigma^2\delta$. Each feature is a kernel
evaluation, and $\lVert\phi(x)\rVert^2\le\sigma^2 k(x,x)$. Unlike random Fourier features
\citep{rahimi2007rff}, which rescale a fixed draw of frequencies and make the likelihood
rugged in the length scales, the features are smooth in every hyperparameter, and they
apply to kernels without a shared spectral representation, such as the non-stationary
kernels of Section~\ref{sec:nonstat}. This parametrisation is non-centred: the prior of
$\xi$ is free of hyperparameters, which enter only through $\phi$ and $A$.

\paragraph{Unbinned likelihood.}
For observation $n$ with $u_n = T(y_n)$, the negative log-likelihood is
\begin{equation}
\ell_n(P_n) = -P_n^\top\bpsi(u_n) + \log\!\int_0^1\!e^{P_n^\top\bpsi(u)}du,
\label{eq:nll}
\end{equation}
evaluated at $u_n$ itself, so no observation is assigned to a bin. The normaliser, the
integral of the exponential of a trigonometric polynomial, is computed by composite
Gauss--Legendre quadrature with $K$ cells of $4$ nodes, a rule whose error is
$O(\Delta^8)$ in the cell width $\Delta$; as $\Delta = 1/K$ we check its accuracy
numerically rather than rely on a rate (Appendix~\ref{app:spectral}). The same nodes give
the predictive density, the CDF (by integrating a polynomial interpolant of the density
within each cell) and every metric. The likelihood is thus evaluated with the quadrature
normaliser $Z_Q$ in place of $Z = \int_0^1e^{g(u)}du$, and the per-object density
$e^{g(u)}/Z_Q$ integrates to $Z/Z_Q$, one up to the quadrature error. What we sample is
therefore the posterior under this numerically normalised likelihood of the finite-feature model,
not that of the unrestricted LGP. Appendix~\ref{app:validation} audits $\log Z_Q$ against a finer
rule at fitted posterior states.

\paragraph{The vanilla model.}
With a stationary kernel $K_x$ and the fixed transform, the hyperparameters are
$\vartheta = (\sigma^2,\ell_x,\ell_y)$: one variance and one length scale per axis. We describe
its inference next. Section~\ref{sec:nonstat} then replaces the
stationary kernel by a deep, non-stationary one and learns the transform; none of the
inference changes.

\section{INFERENCE}
\label{sec:inference}

\subsection{The latent posterior is strongly log-concave}

Fix $\vartheta$. The latent posterior is $\pi(\xi)\propto\exp\{-U(\xi)\}$,
$U(\xi) = \sum_n\ell_n(\phi_n^\top\xi) + \tfrac12\lVert\xi\rVert_F^2$, with
$\phi_n=\phi(x_n)$.

\begin{proposition}\label{prop:lc}
With $\Phi$ the $N\times J$ feature matrix and $s^2 = \sum_kA_k^2$,
\[
I \preceq \nabla^2U(\xi)\preceq\bigl(1+s^2\lVert\Phi\rVert^2\bigr)I\qquad\text{for all }\xi .
\]
\end{proposition}

The first part holds because \eqref{eq:lgp} is an exponential family in its natural
parameter $P_n$, linear in $\xi$; the second because the Hessian of $\log Z$ is the
covariance of the statistic $\bpsi(u)$, and $\lVert\bpsi(u)\rVert^2 = s^2$ for every $u$
(proof in Appendix~\ref{app:proofs}). For the Mat\'ern spectrum $s^2$ is at most the prior
variance of the response process, so it is bounded uniformly in $K$. The numerical floor
$\varepsilon/\sigma^2$ that we add to every mode (Appendix~\ref{app:spectral}) contributes
$(K-1)\varepsilon/\sigma^2$, only $10^{-4}/\sigma^2$ at $K = 100$ but growing with $K$, so when we say
uniform in $K$ we mean the spectrum without the floor. For Nystr\"om features
$\lVert\Phi\rVert^2\le\sigma^2N$. The same cumulant argument bounds
the third derivative in the norm that the minibatch sampler's analysis requires.

\begin{lemma}\label{lem:reg}
For every $\xi$ and unit $v\in\R^{d}$,
$\lVert\nabla^3U(\xi)[v]\rVert_F\le M_1^{\mathrm s} := 8s^3\sum_n\lVert\phi_n\rVert^3\le
8s^3\sigma^3N$. For a partition of the rows into $n_b$ batches, each component
$f_i = \lVert\xi\rVert_F^2/(2n_b)+\sum_{n\in B_i}\ell_n$ is $n_b^{-1}$-strongly convex and
satisfies the upper bound of Proposition~\ref{prop:lc} and this bound with the sums restricted to $B_i$.
\end{lemma}

$M_1^{\mathrm s}$ is the ``strongly Hessian Lipschitz'' constant of
\citet{paulin2024smsubu}, under which the Wasserstein bias of SMS--UBU is $O(h^2\sqrt d)$,
$d = 2J(K-1)$, instead of the $O(h^2d)$ that an ordinary Hessian Lipschitz bound gives. The
dimension does not disappear into the condition number, but the constants of
Lemma~\ref{lem:reg} grow with $N$ and not with $J$ or $K$, and $d\approx2\cdot10^5$ in our
photometric-redshift experiments.

\subsection{SMS--UBU in Kronecker-whitened coordinates}
\label{sec:sampler}

\paragraph{Whitening.}
The raw condition number grows with $N$. The Gauss--Newton Hessian is
$I + \sum_n(\phi_n\phi_n^\top)\otimes\Sigma_n$, where $\Sigma_n$ is the covariance of
$\bpsi$ under row $n$'s density, a softmax covariance over the quadrature nodes. As in
K-FAC \citep{martens2015kfac} we approximate the sum by one Kronecker product at the
average chain position: $G = \Phi^\top\Phi$ and $S = $ the row-averaged $\Sigma_n$, with
eigendecompositions $G = U\Lambda_GU^\top$, $S = V\Lambda_SV^\top$, and run the chains in
coordinates $x$ with $\xi = U(x\oslash s)V^\top$, $s = (\Lambda_G\otimes\Lambda_S+1)^{1/2}$,
in which $I + G\otimes S$ is the identity. Because the latent prior is the identity both
eigenproblems are ordinary, which keeps the scheme accurate in single precision. A rebuild
costs one $J\times J$ and one $2K\times 2K$ eigendecomposition. The factors are rebuilt
every ten optimisation steps (more often early on SDSS and DESI, Appendix~\ref{app:nonstat}) and never during a sampling phase.

\paragraph{The sampler.}
Kinetic Langevin dynamics $d\xi = v\,dt$, $dv = -\nabla U\,dt-\gamma v\,dt +
\sqrt{2\gamma}\,dW$ is discretised by the UBU splitting $\mathcal U(h/2)\mathcal
B(h)\mathcal U(h/2)$, which integrates the Ornstein--Uhlenbeck part exactly and is
second-order accurate for the invariant measure. With symmetric minibatch splitting
(SMS--UBU, \citealp{paulin2024smsubu}) each sweep draws a random partition into $n_b$
batches and applies the batch impulses in the palindromic order $\mathcal B_1,\dots,
\mathcal B_{n_b},\mathcal B_{n_b},\dots,\mathcal B_1$, which cancels the leading
minibatch bias, so $10$ to $40$ large batches suffice. We use friction $\gamma=2$, step
$h = \min\{1/(2\gamma),\,0.5\hat\lambda^{-1/2}\}$ with $\hat\lambda$ the largest
eigenvalue of the whitened Hessian from a power iteration, and four chains in two
antithetic pairs. All chains share the batch partition, so each batch's feature rows
are built only once per impulse.

\paragraph{What is guaranteed.}
Let $R$ be the fixed linear map $\vecop(\xi) = R\,\vecop(x)$ of the whitening,
$\lVert R\rVert\le1$ since $s\ge1$, and define $m_R = \lambda_{\min}(R^\top R)$, the
full-data curvature bound $M_R^{\mathrm{full}} = \lambda_{\max}(R^\top(I + s^2G\otimes I)R)$
and the batch-uniform bound $M_R = \lambda_{\max}(R^\top(I + n_bs^2G\otimes I)R)$. Since the
Gram matrix of any batch satisfies $G_B\preceq G$, every component $f_i$ of every partition
the sampler can draw is $(M_R/n_b)$-smooth in whitened coordinates, whereas
$M_R^{\mathrm{full}}\le M_R$ bounds only their sum. All three constants have closed forms in
terms of $(\Lambda_G,\Lambda_S)$ (Appendix~\ref{app:proofs}).

\begin{theorem}\label{thm:sms}
Fix $\vartheta$ and $R$, and run SMS--UBU on $U(R\,\cdot)$ with $n_b$ batches, friction
$\gamma\ge\sqrt{8M_R}$ and step $h<1/(2\gamma)$. After $k$ sweeps the law $\mu_k$ of the
latent state satisfies
\[
W_2(\mu_k,\pi)\le C_0\,e^{-m_Rhk/(8\gamma)} + C_1h^2\Bigl(\sqrt{d}+\frac{\sqrt{m_R}\,\bar\sigma_*}{M_R}\Bigr),
\]
with $C_0$ depending on the initial law, $C_1 = C(\gamma,m_R,M_R,M_1^{\mathrm s},n_b)$
independent of $J$, and of $K$ up to the floor term in $s^2$, and $\bar\sigma_* = 2sN\max_n\lVert\phi_n\rVert$,
which bounds the spread $\sigma_*(\mathcal P) = (n_b\sum_i\lVert\nabla f_i(\xi^*)\rVert^2)^{1/2}$ of the minibatch gradients at
the posterior mode $\xi^*$ for every partition $\mathcal P$ into $n_b$ equal batches.
\end{theorem}

Theorem~\ref{thm:sms} is Theorem~5 of \citet{paulin2024smsubu} applied to $U(R\,\cdot)$, with the
constants supplied by Lemma~\ref{lem:reg} and $\lVert R\rVert\le1$. That result assumes that the
components share a minimiser, which ours do not, and dropping the assumption adds the
dimension-free term in $\bar\sigma_*$ (Appendices~\ref{app:proofs} and~\ref{app:minimiser}).
Five aspects of what we run fall outside the theorem; we check the first, second and fourth empirically
(Appendix~\ref{app:validation}). First, our default $\gamma=2$ is far below $\sqrt{8M_R}$.
$M_R$ bounds the curvature at every $\xi$ and for every batch, and on the controlled problem of
Appendix~\ref{app:validation} it exceeds the local whitened curvature $\hat\lambda$ by more than
three orders of magnitude. Second, $\hat\lambda$ sets our step but is only an estimate at the
chain mean, so it cannot stand in for $M_R$ in the theorem. Because the tilted covariances
$\Sigma_n$ can collapse, we have no uniform bound on the whitened condition number better than
$M_R/m_R$. Third, $R$ is rebuilt during
the optimisation, so the theorem applies at fixed $\vartheta$ and $R$, not to the
adaptive scheme as a whole. Fourth, it concerns the finite-feature, quadrature-based
posterior and says nothing about the representation error of finite $J$, $K$ and $Q$. Fifth, each
sweep omits up to $112$ random rows (a fraction below $5\cdot10^{-4}$), which the theorem does not cover. Against a $\pi$-invariant reference on a controlled problem
(Appendix~\ref{app:validation}) the default step shows small but detectable biases (latent
variances $0.4\%$, the gradient up to $13$ standard errors), substantially reduced at half the
step with little effect on the predictive NLL.

\subsection{Learning the hyperparameters}
\label{sec:fisher}

The marginal likelihood is intractable, but its gradient obeys Fisher's identity,
\begin{equation}
\nabla_\vartheta\log p(y\mid\vartheta) = \E_{\xi\sim\pi_\vartheta}\Bigl[\nabla_\vartheta
\log p(y\mid\xi,\vartheta)\Bigr],
\label{eq:fisher}
\end{equation}
in the non-centred parametrisation, where no prior normaliser depends on $\vartheta$.
We estimate the expectation from the states of the warm-started chains and follow it by
stochastic approximation \citep[SOUL,][]{debortoli2021soul}; Algorithm~\ref{alg:fit}
summarises one outer step. Per step, each chain $c\le C$ draws a gradient batch $B_c$ of
rows, independently of the sampler's partition, collects the cotangent of the batch
negative log-likelihood with respect to its feature rows at $S$ positions of its
collection sweeps, and pulls the sum back to $\vartheta$ by a single vector--Jacobian
product. The estimate of the gradient of the negative log-posterior of $\vartheta$ is
\[
\hat g = \frac{1}{CS}\sum_{c=1}^{C}\sum_{t=1}^{S}\frac{N}{\lvert B_c\rvert}
\sum_{n\in B_c}\nabla_\vartheta\ell_n(\xi_{c,t};\vartheta) - \nabla_\vartheta\log p(\vartheta).
\]
Here $\ell_n$ includes the Jacobians of the transform and the response warp, and it is
differentiated at fixed non-centred $\xi_{c,t}$; we do not differentiate through the sampler's
trajectory.

\begin{algorithm}[t]
\caption{One outer step of the fit}\label{alg:fit}
\begin{algorithmic}[1]
\Require chains $(\xi_c,v_c)_{c\le4}$, whitening $(U,V,s)$, step $h$, $\vartheta$
\If{step $\equiv 0 \pmod{10}$}
\State rebuild $(U,V,s)$ at $\bar\xi$; power iteration for $h$
\EndIf
\State $n_{\mathrm{burn}}$ SMS--UBU sweeps of every chain
\State draw a gradient batch $B_c$ per chain; $\hat g\gets0$
\For{$n_{\mathrm{coll}}$ collection sweeps, every 5th impulse}
\State $\hat g\gets\hat g + \frac{N}{\lvert B_c\rvert}\nabla_\vartheta\sum_{n\in B_c}\ell_n(\phi_n(\vartheta)^\top\xi_c)$
\EndFor
\State $\hat g\gets\hat g/(CS)-\nabla_\vartheta\log p(\vartheta)$; optimiser step on $\hat g$
\end{algorithmic}
\end{algorithm}

\paragraph{Why sampling rather than Laplace.}
Besides Monte Carlo variance, the estimator of \eqref{eq:fisher} has two sources of bias:
the discretisation bias of the sampler and incomplete mixing of the warm-started chains.
The integrand $\nabla_\vartheta\log p(y\mid\xi,\vartheta)$ is not globally Lipschitz in
$\xi$, but it grows at most linearly in the sup norm of each row's field, whose posterior
fluctuations depend on $J$ and $K$ only through feature norms, their derivatives and
spectral sums (Appendix~\ref{app:bias}). Under the conditions of Theorem~\ref{thm:sms} the gradient
bias is therefore $O(h^2(\sqrt d+\sqrt{m_R}\,\bar\sigma_*/M_R))$ plus a term that decays geometrically
in the number of sweeps, with a prefactor that depends on the posterior-mean field and the
feature derivatives, which we do not bound uniformly in $J$ and $K$. The
Laplace approximation instead differentiates an approximate marginal likelihood,
$\log p(y,\hat\xi_\vartheta\mid\vartheta)-\tfrac12\log\det H_\vartheta(\hat\xi_\vartheta)$
at the mode $\hat\xi_\vartheta$. Its gradient is exact for the approximation, but the
approximation error depends on how non-Gaussian the posterior is, which varies with
$\vartheta$, and the gradient of that error does not shrink with more computation. It also needs the
log-determinant of a $d$-dimensional Hessian and its third derivatives through
$\hat\xi_\vartheta$, whereas \eqref{eq:fisher} needs only first derivatives at sampled
states. On a controlled problem (Appendix~\ref{app:validation}) the two kinds of gradient
lead to similar hyperparameters. The larger difference is in prediction. Under the reported budgets and at the same
hyperparameters, the Gaussian predictive around the mode is $0.12$ nats worse than the sampled
one, a finite-run comparison rather than a converged-posterior gap. The main practical reason to
sample is thus to integrate the predictive over the non-Gaussian latent posterior.

\paragraph{Optimisation.}
For the vanilla model we follow \eqref{eq:fisher} with Adam on the log-hyperparameters.
For the deep model of Section~\ref{sec:nonstat} we use Muon \citep{jordan2024muon} for the
network's weight matrices and Adam for everything else, on a cosine schedule. The deep models
hold out a small validation set from their training rows, track its log-likelihood alongside
the predictive and keep the best state (early stopping). The vanilla fits
use all training rows and run to the end of the schedule. We do not claim convergence of this
outer procedure, whose target and preconditioner change. $\vartheta$ is an empirical-Bayes point
estimate, and the predictive integrates over $\xi$ conditionally on it, not over the
encoder, transform and kernel hyperparameters.

\paragraph{Predictive.}
The predictive is the posterior mean density $\E_\pi p_u(u\mid x;\xi)$, averaged over the
last fifteen of twenty further sweeps of the trained chains. Each state contributes its
density at the quadrature nodes and, at the observation, its density (normalised by $Z_Q$)
and its CDF (from the within-cell interpolant), from which the log-likelihood, PIT, CRPS (in
quantile form), CDE loss and point summaries are computed.

\section{THE VANILLA MODEL ON A BIMODAL SINUSOID}
\label{sec:sinus}

The data are $x\sim U(-1,1)$, $y = x\pm\sin(x/0.15)+\varepsilon$, $\varepsilon\sim
\N(0,0.1^2)$, with the sign drawn uniformly: an equal-weight Gaussian mixture whose modes
cross wherever $\sin(x/0.15) = 0$. We fit the vanilla model with RBF kernels on both
axes, so that only $(\sigma^2,\ell_x,\ell_y)$ are learned, on nested training sets of
$400$, $2{,}000$ and $10{,}000$ rows, and score all fits on the same $2{,}000$ test rows.
We repeat this over four seeds, each splitting one sample of $20{,}000$ draws differently into training and test rows.
Table~\ref{tab:sinus} compares with TabPFN-3.5
\citep{priorlabs2026tabpfn35} given the same training rows as context. The three-parameter
GP has the best NLL and CDE loss at every size. Averaged over the four seeds it is ahead of
TabPFN-3.5 by $0.017$, $0.009$ and $0.016$ nats at the three sizes. It wins nine of the
twelve paired comparisons; in the other three (seed $2$ at $400$ rows, seeds $2$ and $3$ at
$2{,}000$) TabPFN-3.5 is ahead by $0.002$--$0.005$ nats. Its excess over the true
density on the same test rows is $0.238$, $0.073$ and $0.026$ nats (standard deviations over
seeds $0.031$, $0.006$ and $0.001$, below those of the NLL). TabPFN-3.5 has a lower RMSE and CRPS at
$400$ and $2{,}000$ rows, and the two tie at $10{,}000$. The other baselines trail further. TabICL-2 \citep{qu2025tabicl} barely
improves with $N$, Flow-Spline \citep{izbicki2026benchmark} reaches $-0.119$ only at $10{,}000$ rows, and at $400$
rows Flow-Spline and the CRPS ensemble \citep{kelen2025distributionfree} are far behind and unstable across seeds
(standard deviations $0.30$ and $0.38$ nats). Both length scales shrink
as $N$ grows, in every seed. $\ell_x$ goes from $0.163$ (s.d.\ $0.012$) to $0.136$
($0.004$) and $0.119$ ($0.013$), near the scale $0.15$ of the sinusoid, and $\ell_y$ from
$0.132$ ($0.010$) to $0.109$ ($0.002$) and $0.100$ ($0.001$). Each fit takes $4$--$14$\,min
on one GPU.

\begin{table}[!t]
\caption{Bimodal sinusoid: test metrics on $2{,}000$ test rows of the vanilla LGP (RBF
kernels, Nystr\"om features, three learned hyperparameters), TabPFN-3.5 and TabICL-2 with the
training rows as context, Flow-Spline tuned by the protocol of \citet{izbicki2026benchmark},
the CRPS ensemble (batch $32$, early stopping on $10\%$ of the training rows with patience
$200$ epochs), and the true density. All rows are averaged over four seeds, each a different
split of one sample into training and test rows (standard deviations in the rows below). A PIT-KS of about $0.02$ is the sampling noise of $2{,}000$ rows. Bold: best of
the five models.}\label{tab:sinus}
\begin{center}
\footnotesize\setlength{\tabcolsep}{2pt}
\begin{tabular}{lccccc}
\toprule
model & NLL & RMSE & CRPS & CDE & PIT-KS \\
\midrule
\multicolumn{6}{l}{\emph{$N=400$}} \\
LGP (ours) & $\mathbf{-0.0256}$ & $0.7262$ & $0.3635$ & $\mathbf{-1.179}$ & $0.0471$ \\
\quad s.d. & $0.0388$ & $0.0072$ & $0.0074$ & $0.049$ & $0.0268$ \\
TabPFN-3.5 & $-0.0091$ & $0.7085$ & $\mathbf{0.3496}$ & $-1.129$ & $\mathbf{0.0428}$ \\
\quad s.d. & $0.0317$ & $0.0021$ & $0.0023$ & $0.047$ & $0.0258$ \\
TabICL-2 & $+0.2101$ & $0.7079$ & $0.3511$ & $+0.055$ & $0.0448$ \\
\quad s.d. & $0.0292$ & $0.0050$ & $0.0025$ & $0.324$ & $0.0286$ \\
Flow-Spline & $+0.5951$ & $\mathbf{0.7070}$ & $0.3819$ & $-0.642$ & $0.0720$ \\
\quad s.d. & $0.3002$ & $0.0075$ & $0.0173$ & $0.217$ & $0.0326$ \\
CRPS ens. & $+0.4746$ & $0.7110$ & $0.3740$ & $-0.781$ & $0.0963$ \\
\quad s.d. & $0.3781$ & $0.0028$ & $0.0274$ & $0.277$ & $0.0350$ \\
\midrule
\multicolumn{6}{l}{\emph{$N=2{,}000$}} \\
LGP (ours) & $\mathbf{-0.1907}$ & $0.7059$ & $0.3431$ & $\mathbf{-1.405}$ & $0.0315$ \\
\quad s.d. & $0.0153$ & $0.0044$ & $0.0016$ & $0.020$ & $0.0070$ \\
TabPFN-3.5 & $-0.1822$ & $\mathbf{0.7001}$ & $\mathbf{0.3386}$ & $-1.403$ & $\mathbf{0.0234}$ \\
\quad s.d. & $0.0174$ & $0.0052$ & $0.0018$ & $0.030$ & $0.0064$ \\
TabICL-2 & $+0.1309$ & $0.7005$ & $0.3400$ & $-0.820$ & $0.0306$ \\
\quad s.d. & $0.0280$ & $0.0064$ & $0.0027$ & $0.067$ & $0.0068$ \\
Flow-Spline & $+0.0383$ & $0.7022$ & $0.3464$ & $-1.137$ & $0.0466$ \\
\quad s.d. & $0.0516$ & $0.0069$ & $0.0038$ & $0.051$ & $0.0139$ \\
CRPS ens. & $+0.0864$ & $0.7011$ & $0.3455$ & $-1.147$ & $0.0937$ \\
\quad s.d. & $0.0141$ & $0.0068$ & $0.0029$ & $0.017$ & $0.0100$ \\
\midrule
\multicolumn{6}{l}{\emph{$N=10{,}000$}} \\
LGP (ours) & $\mathbf{-0.2385}$ & $0.6998$ & $\mathbf{0.3376}$ & $\mathbf{-1.486}$ & $\mathbf{0.0241}$ \\
\quad s.d. & $0.0152$ & $0.0061$ & $0.0020$ & $0.018$ & $0.0074$ \\
TabPFN-3.5 & $-0.2228$ & $0.6997$ & $\mathbf{0.3376}$ & $-1.469$ & $0.0249$ \\
\quad s.d. & $0.0153$ & $0.0061$ & $0.0019$ & $0.021$ & $0.0078$ \\
TabICL-2 & $+0.1254$ & $\mathbf{0.6991}$ & $0.3380$ & $-0.858$ & $0.0345$ \\
\quad s.d. & $0.0242$ & $0.0066$ & $0.0023$ & $0.031$ & $0.0067$ \\
Flow-Spline & $-0.1188$ & $0.7003$ & $0.3407$ & $-1.321$ & $0.0392$ \\
\quad s.d. & $0.0304$ & $0.0069$ & $0.0031$ & $0.059$ & $0.0182$ \\
CRPS ens. & $+0.0380$ & $0.6994$ & $0.3427$ & $-1.218$ & $0.0822$ \\
\quad s.d. & $0.0149$ & $0.0058$ & $0.0024$ & $0.027$ & $0.0195$ \\
\midrule
true density & $-0.2640$ & $0.6982$ & $0.3359$ & $-1.528$ & $0.0181$ \\
\bottomrule
\end{tabular}
\end{center}
\end{table}

\section{A DEEP, NON-STATIONARY KERNEL}
\label{sec:nonstat}

Real covariates are high-dimensional and the conditional densities vary in scale,
reach and shape across input space. We keep the structure of Section~\ref{sec:model}
and make every ingredient input-dependent through a small network.

\paragraph{Encoder.}
A pre-LayerNorm perceptron $h_\theta$ reads the covariates and returns, through a
residual scale of $0.01$ so that the model starts stationary, a warp
$e_\theta(x) = x + 0.01h_\theta(x)_{1:D}$, a log-amplitude $\alpha_\theta(x)$, a log
multiplier $c_\theta(x)$ of the response length scale, local log length scales
$\lambda_\theta(x)\in\R^D$ and the coefficients $\beta_\theta(x)$ of a response warp. A
learned full-matrix bandwidth $z = e^s(I+B)e_\theta(x)\oslash\ell_x$ is applied after the warp
\citep[deep kernel learning,][]{wilson2016dkl}.

\paragraph{Gibbs kernel.}
The covariate kernel is
\begin{equation}
\begin{aligned}
K_x(x,x') &= \sigma(x)\sigma(x')\prod_{d}\Bigl(\tfrac{2l_dl_d'}{l_d^2+l_d'^2}\Bigr)^{1/2}\\
&\quad\times\kappa_\nu\Bigl(\Bigl[\sum_d\tfrac{(z_d-z_d')^2}{(l_d^2+l_d'^2)/2}\Bigr]^{1/2}\Bigr)
\end{aligned}
\label{eq:gibbs}
\end{equation}
with $\sigma(x)=\sigma e^{\alpha_\theta(x)}$, $l_d = e^{\lambda_{\theta,d}(x)}$, $l_d' =
e^{\lambda_{\theta,d}(x')}$, $z = z(x)$ and $z' = z(x')$. This is an amplitude-modulated
Gibbs--Paciorek kernel \citep{gibbs1997thesis,paciorek2004nonstationary}, positive
definite in any dimension. The warp changes which objects are neighbours, the local
length scales how far an object's influence reaches along each coordinate, and the
amplitude how far its density may depart from uniform. Nystr\"om features
\eqref{eq:nystrom} represent it on the anchors up to the regularisation of
Section~\ref{sec:model}, with the anchors passed through the same encoder. Since $k(x,x) = 1$, we have $\lVert\phi(x)\rVert^2\le\sigma(x)^2$.

\paragraph{Non-stationarity along the response.}
The prior is diagonal in the Fourier basis, so a per-object response length scale
$\ell_ye^{c_\theta(x_n)}$ is a per-object damping of the modes, $d_{nk} =
A_k(\ell_ye^{c_\theta(x_n)})/A_k(\ell_y)$, and $P_n = (\phi_n^\top\xi)\odot(d_n,d_n)$.
A per-object monotone warp $v = w_x(u) = \int_0^ue^{r_x}/\int_0^1e^{r_x}$, $r_x(t) =
\sum_{j\le J_w}\beta_j(x)\cos\pi jt$, makes the process stationary in $v$ rather than
$u$, so that along $u$ its effective length scale is $\ell_ye^{c_\theta(x)}/w_x'(u)$.
One object's density can thus be sharp in one part of its range and smooth in another. The
density becomes $p_u(u\mid x) = p_v(w_x(u)\mid x)w_x'(u)$ with $p_v$ of the form
\eqref{eq:lgp}, so the observation moves to $v_n = w_{x_n}(u_n)$ and the Jacobian
$\log w'_{x_n}(u_n)$ joins the objective.

\paragraph{Learnable transform.}
We also learn the response transform $T$. Its log-derivative at $400$ knots forms a block
of hyperparameters, initialised at the kernel density estimate and given a
second-difference smoothness prior, and its Jacobian $\log T'(y_n)$ joins the objective.

\paragraph{Inference is unchanged.}
None of these changes touches the latent prior. The amplitude, the Gibbs length scales and
the warp act on $\phi_n$, and the damping rescales the coordinates of $P_n$. The response warp and
the transform only move the observation point and add terms free of $\xi$. Hence:

\begin{proposition}\label{prop:lc2}
For the model of this section, Proposition~\ref{prop:lc} and Lemma~\ref{lem:reg} hold with
$s^2 = \max_n\sum_kA_k^2d_{nk}^2$ and $\lVert\phi_n\rVert\le\sigma(x_n)$.
\end{proposition}

The sampler, the whitening (with the root-mean-square damping in $S$) and the Fisher
gradients (now with respect to the encoder, bandwidth, transform and kernel
hyperparameters) are those of Section~\ref{sec:inference}.

\begin{table*}[!t]
\caption{Test metrics on the four photometric-redshift benchmarks (NLL and CDE loss:
lower is better; RMSE and CRPS in redshift units; PIT-KS: Kolmogorov--Smirnov distance of
the PIT from uniform). Bold: best of the five models (all tied values at the displayed precision). Rows followed by an s.d. row are means over four seeds
(five runs for $^\ddagger$), with the standard deviation (s.d.) in that row. The seeds change the validation rows and each model's initialisation (and
Flow-Spline's search), and on SDSS and DESI also the train/test split; TabPFN-3.5 and TabICL-2
on \textsc{Happy} and \textsc{Teddy} do not depend on the seed. TabICL-2 and Flow-Spline on DESI
are single runs on the seed-$0$ split, on which the LGP's NLL is $-2.3795$. $^\ddagger$Mean over
five runs with identical settings on that split, two of which diverge
(Appendix~\ref{app:validation}). $^\ast$A context of $500{,}000$ training rows. $^\dagger$Flow-Spline's exact density on
DESI has narrow spikes that dominate $\int p^2$; on the benchmark's $200$-point grid its CDE
loss is $-13.95$.}\label{tab:main}
\begin{center}
\footnotesize
\setlength{\tabcolsep}{3pt}
\begin{tabular}{llccccccc}
\toprule
benchmark (train / test) & model & NLL & RMSE & CRPS & CDE & PIT-KS & NMAD & outliers \\
\midrule
\textsc{Happy} A$\to$B & LGP (ours) & $-2.1557$ & 0.0506 & 0.0203 & $-12.86$ & 0.0054 & 0.0186 & 0.86\% \\
($74{,}950$ / $74{,}900$) & \quad s.d. & 0.0009 & 0.00008 & 0.00002 & 0.02 & 0.0003 & 0.00003 & 0.01\% \\
& TabPFN-3.5 & $\bm{-2.1632}$ & $\bm{0.0493}$ & $\bm{0.0200}$ & $\bm{-12.90}$ & $\bm{0.0040}$ & $\bm{0.0185}$ & $\bm{0.81\%}$ \\
& TabICL-2 & $-1.9958$ & 0.0497 & 0.0201 & $-8.79$ & 0.0068 & $\bm{0.0185}$ & 0.83\% \\
& Flow-Spline & $-2.1093$ & 0.0522 & 0.0208 & $-12.04$ & 0.0220 & 0.0191 & 0.91\% \\
& \quad s.d. & 0.0049 & 0.00015 & 0.00003 & 0.16 & 0.0101 & 0.00007 & 0.01\% \\
& CRPS ensemble & $-2.1081$ & 0.0509 & 0.0206 & $-12.09$ & 0.0092 & 0.0192 & 0.89\% \\
& \quad s.d. & 0.0040 & 0.00011 & 0.00006 & 0.07 & 0.0034 & 0.00009 & 0.01\% \\
\midrule
\textsc{Teddy} A$\to$B & LGP (ours) & $-2.3281$ & 0.0322 & 0.0152 & $-13.96$ & 0.0042 & $\bm{0.0158}$ & 0.24\% \\
($74{,}309$ / $74{,}557$) & \quad s.d. & 0.0009 & 0.00002 & 0.00001 & 0.02 & 0.0011 & 0.00003 & 0.01\% \\
& TabPFN-3.5 & $\bm{-2.3356}$ & $\bm{0.0320}$ & $\bm{0.0151}$ & $\bm{-14.05}$ & 0.0051 & $\bm{0.0158}$ & 0.25\% \\
& TabICL-2 & $-2.2566$ & 0.0321 & 0.0152 & $-10.38$ & $\bm{0.0035}$ & $\bm{0.0158}$ & 0.25\% \\
& Flow-Spline & $-2.3086$ & 0.0324 & 0.0154 & $-13.77$ & 0.0127 & 0.0160 & 0.25\% \\
& \quad s.d. & 0.0036 & 0.00015 & 0.00006 & 0.05 & 0.0076 & 0.00007 & 0.02\% \\
& CRPS ensemble & $-2.3158$ & 0.0322 & 0.0153 & $-13.85$ & 0.0093 & 0.0159 & $\bm{0.23\%}$ \\
& \quad s.d. & 0.0014 & 0.00002 & 0.00001 & 0.01 & 0.0031 & 0.00006 & 0.03\% \\
\midrule
SDSS DR18 & LGP (ours) & $\bm{-2.1123}$ & 0.0518 & 0.0216 & $\bm{-12.53}$ & $\bm{0.0041}$ & $\bm{0.0196}$ & 0.83\% \\
($450{,}000$ / $50{,}000$) & \quad s.d. & 0.0028 & 0.00028 & 0.00006 & 0.05 & 0.0006 & 0.00014 & 0.05\% \\
& TabPFN-3.5 & $-2.1012$ & $\bm{0.0514}$ & $\bm{0.0215}$ & $-12.21$ & 0.0116 & $\bm{0.0196}$ & $\bm{0.80\%}$ \\
& \quad s.d. & 0.0023 & 0.00015 & 0.00006 & 0.04 & 0.0023 & 0.00011 & 0.05\% \\
& TabICL-2 & $-1.8514$ & 0.0520 & 0.0218 & $-6.24$ & 0.0129 & 0.0199 & 0.84\% \\
& \quad s.d. & 0.0051 & 0.00010 & 0.00004 & 0.11 & 0.0016 & 0.00012 & 0.04\% \\
& Flow-Spline & $-2.0062$ & 0.0550 & 0.0229 & $-10.77$ & 0.0384 & 0.0212 & 0.95\% \\
& \quad s.d. & 0.0069 & 0.00044 & 0.00011 & 0.09 & 0.0230 & 0.00016 & 0.06\% \\
& CRPS ensemble & $-2.0309$ & 0.0529 & 0.0223 & $-11.15$ & 0.0147 & 0.0206 & 0.85\% \\
& \quad s.d. & 0.0071 & 0.00038 & 0.00011 & 0.09 & 0.0067 & 0.00012 & 0.04\% \\
\midrule
DESI DR1 BGS & LGP (ours) & $\bm{-2.3786}$ & $\bm{0.0520}$ & $\bm{0.0168}$ & $\bm{-15.95}$ & $\bm{0.0013}$ & $\bm{0.0167}$ & $\bm{0.35\%}$ \\
($3.91$M / $434{,}505$) & \quad s.d. & 0.0010 & 0.00027 & 0.00004 & 0.04 & 0.0004 & 0.00003 & 0.01\% \\
& TabPFN-3.5$^\ast$ & $-2.3564$ & 0.0521 & 0.0169 & $-15.28$ & 0.0078 & 0.0169 & 0.36\% \\
& \quad s.d. & 0.0012 & 0.00025 & 0.00004 & 0.03 & 0.0011 & 0.00003 & 0.01\% \\
& TabICL-2$^\ast$ & $-2.1833$ & 0.0530 & 0.0173 & $-6.38$ & 0.0145 & 0.0173 & 0.36\% \\
& Flow-Spline & $-2.2682$ & 0.0534 & 0.0174 & $+8.71^\dagger$ & 0.0107 & 0.0174 & 0.37\% \\
& CRPS ensemble$^\ddagger$ & $-2.2719$ & 0.0530 & 0.0176 & $-14.05$ & 0.0155 & 0.0180 & $\bm{0.35\%}$ \\
& \quad s.d. & 0.0143 & 0.00014 & 0.00014 & 0.25 & 0.0038 & 0.00020 & 0.01\% \\
\bottomrule
\end{tabular}
\end{center}
\end{table*}

\section{PHOTOMETRIC-REDSHIFT EXPERIMENTS}
\label{sec:experiments}

\paragraph{Benchmarks.}
\textsc{Happy} and \textsc{Teddy} \citep{beck2017teddy} are SDSS catalogues released as
samples A (training) and B (drawn from the same distribution). We train on A and test on
B ($75{,}000$ galaxies each). The SDSS DR18 sample \citep{izbicki2026benchmark} is split
$450{,}000/50{,}000$, and the DESI DR1 Bright Galaxy Survey \citep{desi2025dr1} $3.91$
million$/434{,}505$. The covariates are one magnitude, four colours and their five
per-object errors, prepared with training statistics only (Appendix~\ref{app:data}).

\paragraph{Models.}
The GP uses $J = 1{,}000$ anchors, $K = 100$ modes, Mat\'ern-$3/2$ kernels and Muon for at
most $1{,}000$ steps with early stopping on a held-out $5\%$ (\textsc{Happy},
\textsc{Teddy}), $2\%$ (SDSS) or $0.5\%$ (DESI) of the training rows. The network grows
with the data, from $3\times32$ to $6\times64$ and $10\times128$ (Table~\ref{tab:config}).
TabPFN-3.5 \citep{priorlabs2026tabpfn35} takes the whole training set as context, except
on DESI, where $500{,}000$ rows is the most that fits in 32\,GB. It is scored on its bar
distribution, for which it maps the members' bin probabilities to shared bins by differencing a
single-precision CDF; we floor them at that resolution (Appendices~\ref{app:config}
and~\ref{app:validation}). The CRPS ensemble \citep{kelen2025distributionfree}, a
purpose-built conditional density estimator trained by a strictly proper scoring rule
\citep{gneiting2007strictly}, uses the same inputs and validation fractions. TabICL-2
\citep{qu2025tabicl} takes the same context as TabPFN-3.5 and is scored on its
$999$-quantile predictive distribution. Flow-Spline, the conditional neural spline flow
\citep{durkan2019neural} of the benchmark of \citet{izbicki2026benchmark}, is tuned by that
benchmark's random search and scored on its exact density. All models are scored on the
same test rows by the same code.

\paragraph{Results.}
In Table~\ref{tab:main} the comparison with TabPFN-3.5 divides along the training-set size.
On \textsc{Happy} and \textsc{Teddy}, with $75{,}000$ training galaxies, TabPFN-3.5 is ahead by $0.008$ nats on both, eight times the
standard deviation of the LGP over seeds, with up to $3\%$ lower RMSE. On \textsc{Teddy} the
LGP's PIT statistic is lower ($0.0042\pm0.0011$ against $0.0051$). On SDSS the
ordering of the density metrics reverses: the LGP is ahead on all four splits, by $0.011$
nats on average (standard deviation of the difference $0.0007$) and $0.32$ units of CDE
loss, with a PIT statistic of $0.004$ against $0.012$, while TabPFN keeps an edge of under
$1\%$ in RMSE and CRPS. On DESI, where TabPFN-3.5 can use an eighth of the
training set, the LGP leads on every metric and on all four splits, by $0.022$ nats on
average (standard deviation of the difference $0.0004$) and $0.67$ units of CDE loss, with a
PIT statistic of $0.001$ against $0.008$, essentially uniform on $434{,}505$ test galaxies.
At equal data the two are close. Trained on TabPFN-3.5's $500{,}000$ context rows, the GP with a
$10\times64$ network is ahead by only $0.003$ nats, so most of the lead comes from using the whole
training set (Appendix~\ref{app:validation}). Switching the latent field off costs $0.13$--$0.37$ nats on the four benchmarks and most
of the calibration, and on a controlled problem a Laplace fit's Gaussian predictive is $0.12$
nats worse than the sampled one in finite runs. The CRPS
ensemble, trained on a distributional score, is close to the LGP in RMSE and CRPS but $0.01$ to
$0.11$ nats behind it in NLL and up to twelve times less well calibrated. TabICL-2 is close to the LGP and TabPFN-3.5 in
RMSE and CRPS, but its quantile-based density is $0.07$ to $0.26$ nats behind the LGP with a
much worse CDE loss, and Flow-Spline is $0.02$ to $0.11$ nats behind, close to the CRPS
ensemble. Appendix~\ref{app:slgp} adds the closest model, the spatial logistic GP, run with its
authors' package and tuned on held-out data. It is level with our model without the encoder on its own temperature application and,
in an exploratory comparison, on the smallest sinusoid training sets, about $0.01$ nats behind at larger ones, and $0.2$--$0.3$ nats behind on photometric data.

\begin{table}[t]
\caption{GP configuration and cost on one RTX~5090 (32\,GB), scoring included.}\label{tab:config}
\begin{center}
\footnotesize
\setlength{\tabcolsep}{2.5pt}
\begin{tabular}{lccccc}
\toprule
& network & $J_w$ & sweeps & batches & time \\
\midrule
\textsc{Happy}, \textsc{Teddy} & $3\times32$ & 6 & $2+4$ & 10 & 6 min \\
SDSS & $6\times64$ & 6 & $1+2$ & 20 & 16 min \\
DESI & $10\times128$ & 10 & $1+2$ & 40 & 2.2--2.5\,h \\
\bottomrule
\end{tabular}
\end{center}
\end{table}

\paragraph{What the fit looks like.}
On DESI the response warps stretch each
object's redshift axis locally by up to a factor $2.9$--$3.7$ (median over objects, range over
the four seeds) and compress it elsewhere to $0.52$--$0.61$, and the learned transform changes
the log-density most at $z=0$. The validation optimum is at the last steps in every DESI seed and the test NLL is still
improving, so these fits are limited by the budget rather than the data (Appendix~\ref{app:extra}).

\paragraph{Scale.}
At $3.9$ million rows the $N\times(J+K+J_w)$ feature matrix ($17$\,GB) is not kept on
the GPU. Features of a batch are rebuilt inside each impulse, once for all chains. The
preconditioner statistics and the power iteration run over row chunks, the Gibbs kernel
is evaluated in chunks with rematerialisation, and the validation and test predictive is
formed chunk by chunk from stored latent states (Appendix~\ref{app:scale}). A DESI step
takes about $8$\,s.

\section{DISCUSSION}

A finite-feature logistic Gaussian process can be fitted at scale by sampling its latent
posterior instead of replacing it by a Gaussian. In the spectral, feature-based
representation, and for every input-dependent extension we use, the latent posterior is
strongly log-concave, and its second and third derivatives are bounded by constants that do not
depend on the numbers of features and modes (up to a numerical floor on the spectrum). Under the
conditions of Theorem~\ref{thm:sms} a kinetic Langevin sampler in Kronecker-whitened coordinates
then has an $O(h^2\sqrt d)$ bias, and Fisher's identity turns its samples into hyperparameter
gradients whose bias comes from the sampler rather than from a Gaussian approximation. The
resulting estimator beats a tabular foundation model in NLL, CDE loss and calibration on the two larger benchmarks, and trails it by under $0.01$ nats on the two smaller ones. The guarantee, however, covers a more conservative
friction and step than our default, a fixed preconditioner and no omitted rows. The network and the Nystr\"om
basis add a non-convex outer problem, solved by stochastic optimisation with early
stopping, for which we claim no convergence. The uncertainty is conditional on the learned
hyperparameters. Under covariate shift the latent field shrinks away from the anchors, the LGP falls
behind TabPFN-3.5 and fails on \textsc{Teddy} A$\to$D, and a few test redshifts fall outside its bounded
response support, where its density is zero (Appendix~\ref{app:validation}). All our large-scale
evidence is from photometric redshifts.

\subsection*{AI Use Statement}
In this work, we used generative AI tools (a large language model coding assistant) for
writing, refactoring and debugging research code, for running and monitoring
experiments on remote GPU machines, and for drafting and editing parts of the manuscript
text. We did not use generative AI tools to generate data, to fabricate or select
results, or as a subject of the study. All AI-assisted code was reviewed and its outputs
were checked against independent implementations or numerical tests; all reported
numbers come from logged runs. We have reviewed all AI-assisted work and take
responsibility for the final content of this work, including text, claims and artifacts
produced with the aid of generative AI.

\bibliography{references}

\clearpage
\appendix
\thispagestyle{empty}
\onecolumn

\section{PROOFS}
\label{app:proofs}

\paragraph{Proof of Proposition~\ref{prop:lc}.}
$P_n = \phi_n^\top\xi$ is linear in $\xi$ and $-P_n^\top\bpsi(u_n)$ is linear in $P_n$.
The log-normaliser $\log Z_n(P) = \log\int_0^1\exp\{P^\top\bpsi(u)\}du$ is the log-partition
function of an exponential family with sufficient statistic $\bpsi(u)$, hence convex in
$P$ (exactly so under the quadrature, which replaces the integral by a positive
combination of exponentials of affine functions of $P$). Convexity survives the linear
map and the sum over $n$, and the prior adds $I$, so $\nabla^2U\succeq I$.

For the upper bound, $\lVert\bpsi(u)\rVert^2 = \sum_kA_k^2(\cos^2\pi ku+\sin^2\pi ku) = s^2$
for every $u$. The Hessian of $\log Z_n$ in $P_n$ is the covariance $\Sigma_n$ of the
sufficient statistic under the row density $p_n$ (the discrete law on the quadrature nodes
with weights $\propto w_qe^{P_n^\top\bpsi(u_q)}$). For a unit vector $a$,
$\Var_{p_n}\{a^\top\bpsi(u)\}\le\sup_u(a^\top\bpsi(u))^2\le s^2$, so $\Sigma_n\preceq s^2I$.
By the chain rule
$\nabla^2U = I+\sum_n(\phi_n\phi_n^\top)\otimes\Sigma_n\preceq I+s^2(\Phi^\top\Phi)\otimes I$,
and $\lambda_{\max}(\Phi^\top\Phi) = \lVert\Phi\rVert^2$. Finally
$s^2 = \sum_{k=1}^{K-1}\hat c(\pi k)\le\sum_{k\ge1}\hat c(\pi k)$, the variance of the
periodised process minus its constant mode, finite and free of $K$. \qed

\paragraph{Proof of Lemma~\ref{lem:reg}.}
Since $P_n^\top\bpsi(u) = \langle\xi,\phi_n\bpsi(u)^\top\rangle_F$, row $n$ contributes
$\ell_n(\xi) = -\langle\xi,W_n(u_n)\rangle + \log\sum_qw_q\exp\langle\xi,W_n(u_q)\rangle$ with
$W_n(u) = \phi_n\bpsi(u)^\top$, the log-partition function of the discrete exponential family
$\rho_n(u_q)\propto w_qe^{\langle\xi,W_n(u_q)\rangle}$ minus a linear term. Its third derivative
is the third central moment of $W_n$ under $\rho_n$:
$\nabla^3\ell_n(\xi)[a,b,c] = \E_{\rho_n}[\langle a,\bar W\rangle\langle b,\bar W\rangle\langle
c,\bar W\rangle]$, $\bar W = W_n-\E_{\rho_n}W_n$. For a unit $v$, $\nabla^3\ell_n(\xi)[v] =
\E_{\rho_n}[\langle v,\bar W\rangle\bar W\bar W^\top]$, whose Frobenius norm is at most
$\E_{\rho_n}[\lvert\langle v,\bar W\rangle\rvert\lVert\bar W\rVert^2]\le\sup\lVert\bar W\rVert^3
\le(2s\lVert\phi_n\rVert)^3$, because $\lVert W_n(u)\rVert_F = \lVert\phi_n\rVert\,\lVert\bpsi(u)
\rVert = s\lVert\phi_n\rVert$. The prior is quadratic, so the triangle inequality gives
$M_1^{\mathrm s}$. For Nystr\"om features $\lVert\phi_n\rVert\le\sigma$ (see below). A component
$f_i$ carries the prior share $\lVert\xi\rVert^2/(2n_b)$ and the rows of $B_i$, and the same
arguments apply to it. \qed

\paragraph{Proof of Theorem~\ref{thm:sms} and the constants.}
The whitening is $\vecop(\xi) = R\,\vecop(x)$ with $R = (U\otimes V)\diag(s)^{-1}$, an orthogonal
matrix times a diagonal with entries $s_{ij}^{-1}\le1$, so $\lVert R\rVert\le1$. For
$\tilde U = U\circ R$: $\nabla^2\tilde U = R^\top\nabla^2UR\succeq R^\top R\succeq m_RI$ with
$m_R = \min_{ij}s_{ij}^{-2} = (1+\lambda_{\max}(G)\lambda_{\max}(S))^{-1}$;
$\nabla^2\tilde U\preceq R^\top(I+s^2G\otimes I)R$, which is diagonal in the $(U\otimes V)$ basis,
so its largest eigenvalue is $M_R^{\mathrm{full}} = \max_{ij}(1+s^2\lambda_{G,i})/(1+\lambda_{G,i}\lambda_{S,j})$;
and $\nabla^3\tilde U(x)[v] = R^\top(\nabla^3U(Rx)[Rv])R$ has Frobenius norm at most
$\lVert R\rVert^3M_1^{\mathrm s}\le M_1^{\mathrm s}$. Theorem~5 of \citet{paulin2024smsubu} also asks
each of the $n_b$ components to be $(m_R/n_b)$-strongly convex, which the prior share gives,
and $(M_R/n_b)$-smooth. In whitened coordinates the Hessian of $f_i$ is at most
$R^\top(n_b^{-1}I+s^2G_{B_i}\otimes I)R$ with $G_{B_i} = \sum_{n\in B_i}\phi_n\phi_n^\top\preceq G$,
hence at most $n_b^{-1}R^\top(I+n_bs^2G\otimes I)R$, whose largest eigenvalue is
$M_R/n_b$ with $M_R = \max_{ij}(1+n_bs^2\lambda_{G,i})/(1+\lambda_{G,i}\lambda_{S,j})$. This
bound holds for every partition, so it covers the random partitions redrawn at each sweep.
The implementation partitions a random $n_bN_B\le N$ of the rows into equal batches of $N_B$
rows and rescales the impulses by $N/N_B$, omitting a remainder of $r<N_B$ rows ($r = 2$, $34$, $40$ and
$112$ on \textsc{Happy}, \textsc{Teddy}, SDSS and DESI; $0$ in Appendix~\ref{app:validation}). For a uniformly
sampled batch $B$, $\hat g_B(\xi) = \xi+(N/N_B)\sum_{n\in B}\nabla\ell_n(\xi)$ satisfies $\E_B\hat g_B(\xi) = \nabla U(\xi)$
at a fixed $\xi$. This does not, however, make the omitted-remainder sampler the one the theorem analyses, and the
omitted fraction $r/N\le5\cdot10^{-4}$ of the rows is not by itself a bound on the resulting error.
The full-data constant $M_R^{\mathrm{full}} = \max_{ij}(1+s^2\lambda_{G,i})/(1+\lambda_{G,i}\lambda_{S,j})$
of the display above bounds only the sum of the components. Theorem~5 of
\citet{paulin2024smsubu} further assumes that the components share the minimiser of the sum.
Ours do not, and Theorem~\ref{thm:nocommon} (Appendix~\ref{app:minimiser}) proves the bound
without this assumption, with $\sqrt d$ replaced by $\sqrt d+\sqrt{m_R}\,\sigma_*/M_R$, where
$\sigma_*$ is the spread of the whitened components $f_i(R\,\cdot)$, whose gradients are $R^\top\nabla f_i$. Since
$\lVert R\rVert\le1$, this is at most the spread $\sigma_*(\mathcal P)$ of the $f_i$, which Corollary~\ref{cor:sigma} bounds
by $\bar\sigma_*$ for every partition $\mathcal P$. The bound therefore holds uniformly over the partitions redrawn at each sweep. With these constants
we obtain the stated bound. Since $\lVert R\rVert\le1$, Wasserstein distances in $\xi$ are at most those in
$x$. The dimension dependence of the related analysis of \citet{sanzserna2021wasserstein}
requires the stronger assumption identified by \citet{paulin2024correction}; the
strongly-Hessian-Lipschitz condition of Lemma~\ref{lem:reg} is of that kind. \qed

\paragraph{Nystr\"om feature norm.}
With $\phi(x) = \sigma L^{-1}k_a(x)$, $k_a(x) = (k(\bar x_j,x))_j$, $\lVert\phi(x)\rVert^2
= \sigma^2k_a(x)^\top(K_{aa}+\delta I)^{-1}k_a(x)\le\sigma^2k_a(x)^\top K_{aa}^{+}k_a(x)
\le\sigma^2k(x,x)$, the last step because the Nystr\"om approximation is the projection of
$k(\cdot,x)$ onto the span of the anchors in the reproducing kernel Hilbert space. Hence
$\lVert\Phi\rVert^2\le\lVert\Phi\rVert_F^2\le\sigma^2\sum_nk(x_n,x_n)$.

\paragraph{Proof of Proposition~\ref{prop:lc2}.}
The per-object damping gives $P_n = \phi_n^\top\xi D_n$ with $D_n$ diagonal and positive,
still linear in $\xi$. The response warp replaces $u_n$ by $v_n = w_{x_n}(u_n)$, which
does not depend on $\xi$, and the Jacobians $\log w'_{x_n}(u_n)$, $\log T'(y_n)$ are
constant in $\xi$. The arguments above apply row by row with the statistic
$\bpsi_n(u) = (A_kd_{nk}\cos\pi ku,-A_kd_{nk}\sin\pi ku)_k$, of squared norm
$\sum_kA_k^2d_{nk}^2\le s^2$, and the features $\phi(x_n)$ of the Gibbs kernel, whose
diagonal is $\sigma(x)^2$. Since $A_kd_{nk} = A_k(\ell_ye^{c_\theta(x_n)})$ with
$\lvert c_\theta\rvert\le3$, $s^2$ is bounded by the variance of the response process at
length scale $e^{-3}\ell_y$, again free of $K$. \qed

\section{SMS--UBU WITHOUT A COMMON MINIMISER}
\label{app:minimiser}

Theorem~5 of \citet{paulin2024smsubu} assumes that the components $f_i$ of the potential share
the minimiser of $f=\sum_if_i$, which our minibatch components do not. We show that the
assumption can be dropped at the cost of one dimension-free term.

\paragraph{Setting.} $f=\sum_{i=1}^{N_m}f_i$ on $\R^d$ is $m$-strongly convex and
$M$-$\nabla$Lipschitz with minimiser $x^*$; each $f_i$ is $m_i$-strongly convex and
$(M/N_m)$-$\nabla$Lipschitz; $\pi\propto e^{-f}$. SMS--UBU uses the stochastic gradient
$G(y,\omega)=N_m\nabla f_\omega(y)$, $\omega$ uniform on $\{1,\dots,N_m\}$, in a palindromic sweep,
and $D(y,\omega) = G(y,\omega)-\nabla f(y)$. Only $\sum_i\nabla f_i(x^*) = 0$ is assumed. The
gradient spread at the mode,
\[
\sigma_*^2 = \E_\omega\lVert G(x^*,\omega)\rVert^2 = N_m\sum_{i=1}^{N_m}\lVert\nabla f_i(x^*)\rVert^2 ,
\]
is the variance of the stochastic gradient there. It vanishes exactly when the minimiser is common.

\begin{lemma}\label{lem:mom}
If $X\sim\pi$ then $\E\lVert X-x^*\rVert^2\le d/m$.
\end{lemma}
\begin{proof}
Integration by parts gives $\E\langle\nabla f(X),X-x^*\rangle = d$; strong convexity and
$\nabla f(x^*) = 0$ give $\langle\nabla f(x),x-x^*\rangle\ge m\lVert x-x^*\rVert^2$.
\end{proof}

\begin{lemma}\label{lem:gradmom}
For every $i$ and every square-integrable random vector $Y$,
$\lVert N_m\nabla f_i(Y)\rVert_{L^2}\le N_m\lVert\nabla f_i(x^*)\rVert+M\lVert Y-x^*\rVert_{L^2}$; in
particular $\lVert N_m\nabla f_i(X)\rVert_{L^2}\le N_m\lVert\nabla f_i(x^*)\rVert+M\sqrt{d/m}$ for
$X\sim\pi$. If each index occurs at most $r$ times in $i_1,\dots,i_k$, then
$\sum_jN_m\lVert\nabla f_{i_j}(x^*)\rVert\le rN_m\sigma_*$.
\end{lemma}
\begin{proof}
$\lVert\nabla f_i(y)\rVert\le\lVert\nabla f_i(x^*)\rVert+(M/N_m)\lVert y-x^*\rVert$; take $L^2$ norms and
use Lemma~\ref{lem:mom}. For the sum, $\sum_i\lVert\nabla f_i(x^*)\rVert\le\sqrt{N_m}\,
(\sum_i\lVert\nabla f_i(x^*)\rVert^2)^{1/2}=\sigma_*$.
\end{proof}

\begin{lemma}\label{lem:sgnoise}
For every $y$, $(\E_\omega\lVert D(y,\omega)\rVert^2)^{1/2}\le2M\lVert y-x^*\rVert+\sigma_*$.
\end{lemma}
\begin{proof}
$D(y,\omega) = [N_m(\nabla f_\omega(y)-\nabla f_\omega(x^*))-(\nabla f(y)-\nabla f(x^*))]+N_m\nabla f_\omega(x^*)$;
the bracket has norm at most $2M\lVert y-x^*\rVert$ for every $\omega$, the last term has
$L^2(\omega)$ norm $\sigma_*$.
\end{proof}

With a common minimiser, Lemma~\ref{lem:sgnoise} is Assumption~2 of \citet{paulin2024smsubu}
with $C_{SG} = 2$, and Lemma~\ref{lem:gradmom} is the $L^2$ form of their Lemma~9.

\begin{theorem}\label{thm:nocommon}
Under the assumptions of Theorem~5 of \citet{paulin2024smsubu} ($h<1/(2\gamma)$,
$\gamma\ge\sqrt{8M}$, $f$ $M_1^{\mathrm s}$-strongly Hessian Lipschitz) but without a common
minimiser of the components, the law $\tilde\pi_k$ of SMS--UBU satisfies
\[
\mathcal W_2(\tilde\pi_k,\pi)\le\sqrt2\,e^{-\frac{mh}{8\gamma}\lfloor k/N_m\rfloor}
\mathcal W_{2,a,b}(\bar\pi_0,\bar\pi)+C(\gamma,m,M,M_1^{\mathrm s},N_m)\,h^2\Bigl(\sqrt d+\frac{\sqrt m\,\sigma_*}{M}\Bigr),
\]
with a constant $C$ of the same form as theirs, depending on $(\gamma,m,M,M_1^{\mathrm s},N_m)$, whose
coefficients we do not track; under $M_1$-Hessian Lipschitz continuity only, $h^2d$ becomes
$h^2(d+\sqrt m\,\sigma_*/M)$.
\end{theorem}

\begin{proof}
Their proof combines the contraction of SMS--UBU between two copies driven by the same
partitions and noise (their Proposition~7, applied step by step with the constants $m_i$)
with the bias of SMS--UBU started at stationarity (their Theorem~14), which rests on the
interpolation Lemma~13, the cycle estimate Proposition~12, Lemma~11 and Theorem~6 for the
stochastic-gradient gap, Lemma~10 (sampling without replacement) and the full-gradient
Proposition~8. The minimisers of the components enter in exactly three places, always through
$\lVert N_m\nabla f_i(\cdot)\rVert$ or $\lVert D(\cdot,\omega)\rVert$; the contraction, Lemmas~10 and~13,
Proposition~8 and the It\^o--Taylor terms $I_2,\dots,I_5$ of Proposition~12, which contain
$\nabla^2f_i$, $\nabla^3f_i$ or the full gradient $\nabla f$, do not involve them.

(i) In Theorem~6 and Lemma~11 the stochastic-gradient gap over $k\le2N_m$ steps is controlled by
$(\sum_{i<k}\lVert D(y_i,\omega_i)\rVert_{L^2}^2)^{1/2}$, bounded there by $C_{SG}M\sqrt k\,
(\max_i\lVert y_i-X_{(i+1/2)h}\rVert_{L^2}+\sqrt{d/m})$. By Lemma~\ref{lem:sgnoise} and Minkowski's
inequality in $\ell^2$ it is at most this with $C_{SG}=2$ plus $\sqrt k\,\sigma_*$. Since
$M\sqrt{d/m}+\sigma_* = (M/\sqrt m)(\sqrt d+\sqrt m\,\sigma_*/M)$, Lemma~11 holds with its factor
$\sqrt{N_md(\gamma^4h^2/m^2+1/m)}$ replaced by $\sqrt{N_m(\gamma^4h^2/m^2+1/m)}\,(\sqrt d+\sqrt m\,\sigma_*/M)$.

(ii) In Proposition~12 the terms $I_1^{i,j}$ of the velocity component and their analogue in
the position component have integrands $E(\cdot)\,N_m\nabla f_i(X_s)$ with $X_s\sim\pi$, bounded
there by $O(h^2N_m\gamma^2M\sqrt{d/m})$ through their Lemma~9. By Lemma~\ref{lem:gradmom} the factor
$M\sqrt{d/m}$ becomes $N_m\lVert\nabla f_i(x^*)\rVert+M\sqrt{d/m}$; each component occurs once in
the forward and once in the backward sweep of a cycle, so summing over the cycle gives
$N_mM\sqrt{d/m}+N_m\sigma_*$ in place of $N_mM\sqrt{d/m}$.

(iii) In Theorem~14, for the at most $2N_m-1$ steps of an incomplete cycle, the one-step source
$hM\sqrt{d/m}$ of their Gr\"onwall argument becomes $h(N_m\lVert\nabla f_{\omega_k}(x^*)\rVert+M\sqrt{d/m})$,
whose sum over the remaining steps is at most $2hN_m(M\sqrt{d/m}+\sigma_*)$ by
Lemma~\ref{lem:gradmom} with $r = 2$, against $2hN_mM\sqrt{d/m}$ there.

In each case $M\sqrt{d/m}$ becomes $M\sqrt{d/m}+\sigma_*$, i.e.\ $\sqrt d$ becomes
$\sqrt d+\sqrt m\,\sigma_*/M$ in the terms that come from moments of gradients; the other terms
are unchanged, and bounding every term by the larger factor gives the result.
\end{proof}

The additional term does not depend on the dimension. It is the variance of the stochastic
gradient at the mode, which symmetric splitting does not cancel at first order. A centred
(control-variate) gradient would remove it, but that would be a different algorithm.

\begin{corollary}\label{cor:sigma}
For the latent potential split into $n_b$ batches of equal size with components
$f_i = \lVert\xi\rVert_F^2/(2n_b)+\sum_{n\in B_i}\ell_n$, $\sigma_*\le\bar\sigma_* := 2sN\max_n\lVert\phi_n\rVert$
for every partition, and for a uniformly random partition
$\E\sigma_*^2 = (n_b-1)\frac{N}{N-1}\sum_n\lVert \tilde g_n\rVert^2\le4(n_b-1)\frac{N}{N-1}s^2\sum_n\lVert\phi_n\rVert^2$,
where $\tilde g_n = \nabla\ell_n(\xi^*)-N^{-1}\sum_{n'}\nabla\ell_{n'}(\xi^*)$.
\end{corollary}
\begin{proof}
At the mode $\xi^* = -\sum_n\nabla\ell_n(\xi^*)$, so with equal batch sizes
$\nabla f_i(\xi^*) = \xi^*/n_b+\sum_{n\in B_i}\nabla\ell_n(\xi^*) = \sum_{n\in B_i}\tilde g_n$. Since
$\nabla\ell_n(\xi) = \phi_n(\E_n\bpsi-\bpsi(u_n))^\top$ and $\lVert\bpsi(u)\rVert = s$,
$\lVert\nabla\ell_n\rVert_F\le2s\lVert\phi_n\rVert$, and centring does not increase $\sum_n\lVert \tilde g_n\rVert^2$.
Cauchy--Schwarz gives $\lVert\nabla f_i(\xi^*)\rVert^2\le(N/n_b)\sum_{n\in B_i}\lVert \tilde g_n\rVert^2$, hence
$\sigma_*^2 = n_b\sum_i\lVert\nabla f_i(\xi^*)\rVert^2\le N\sum_n\lVert \tilde g_n\rVert^2\le4s^2N^2\max_n\lVert\phi_n\rVert^2$.
For a uniformly random partition, each batch sum is a simple random sample without replacement
of size $b = N/n_b$ from the centred $(\tilde g_n)$, with $\E\lVert\sum_{n\in B_i}\tilde g_n\rVert^2 =
b\frac{N-b}{N-1}\frac1N\sum_n\lVert \tilde g_n\rVert^2$; multiplying by $n_b^2$ gives the formula.
\end{proof}

\section{THE SPECTRAL RESPONSE AXIS}
\label{app:spectral}

By Poisson summation, $\sum_nc(d+2n) = \tfrac12\hat c(0) + \sum_{k\ge1}\hat c(\pi k)\cos(\pi kd)$
with $\hat c$ the Fourier transform of $c$, i.e.\ the Mat\'ern spectral density. Since
$\cos\pi k(u-u') = \cos\pi ku\cos\pi ku' + \sin\pi ku\sin\pi ku'$, the process
\eqref{eq:series} with $A_k^2 = \hat c(\pi k)$ has exactly this covariance (up to the
omitted constant mode). A white-noise floor $A_k^2\mathrel{+}=\varepsilon/\sigma^2$,
$\varepsilon = 10^{-6}$, keeps every mode's variance positive. It adds $(K-1)\varepsilon/\sigma^2$
to $s^2$, a numerical stabilisation that lies outside the uniform-in-$K$ statements of
Section~\ref{sec:inference}. Sampling the wrapped kernel
$c(\min(d,2-d))$ on a grid instead gives a kink at the antipode and negative eigenvalues
once $\ell_y$ is not small against the period. The exact spectrum is never negative. For
$\nu = 3/2$, $A_k\propto(2\nu/\ell_y^2+\pi^2k^2)^{-(\nu+1/2)/2}$; with the learned
$\ell_y\approx0.13$--$0.25$, $K=100$ keeps the last mode's amplitude below
$2\cdot10^{-3}$ of the largest. The normaliser uses composite Gauss--Legendre
quadrature with $K$ cells of four nodes ($Q = 4K$ nodes). Each cell's rule is exact for
polynomials of degree seven, so the error is $O(\Delta^8\max\lvert f^{(8)}\rvert)$ with
$\Delta = 1/K$. Because the derivatives of the integrand grow with $K$, this bound does not shrink as
$K$ grows, so we check the accuracy numerically. Against a rule four times finer, the
log-normaliser agrees to $4\cdot10^{-7}$ at the learned length scales. This check uses the
global $\ell_y$. Under the per-object damping of Section~\ref{sec:nonstat} the effective
length scale is longer or shorter by at most $e^3$, and a shorter one makes the density
locally sharper. That case is checked in Appendix~\ref{app:validation}. The CDF at an observation is the cumulative
mass of the cells below it plus the integral of the Lagrange interpolant within its cell.

\section{HYPERPARAMETER GRADIENTS: SAMPLING VERSUS LAPLACE}
\label{app:bias}

Let $f(\xi) = \nabla_\vartheta\log p(y\mid\xi,\vartheta) = -\sum_n\nabla_\vartheta\ell_n(\xi)$,
so that the exact gradient is $\E_\pi f$, and let $\mu$ be the law of the sampled states. We
work coordinate by coordinate, $\theta$ a scalar component of $\vartheta$. Row $n$ enters through
its field $g_n(u;\xi) = \langle\phi_n^\top\xi,\bpsi_n(u)\rangle$, with $\bpsi_n$ the damped statistic of the proof of
Proposition~\ref{prop:lc2} (Appendix~\ref{app:proofs}), and
$\ell_n(\xi) = -g_n(v_n;\xi)+\log\int_0^1e^{g_n(u;\xi)}du$ plus terms free of $\xi$, with
$v_n$ the observation point (the integral is the quadrature sum in practice, with positive
weights; nothing below changes). Hence
\begin{equation}
\partial_\theta\ell_n(\xi) = -\dot g_n(v_n;\xi)-g_n'(v_n;\xi)\,\partial_\theta v_n
+\E_{\rho_n(\xi)}\dot g_n(\cdot\,;\xi) + (\text{terms free of }\xi),
\label{eq:dtheta}
\end{equation}
with $\dot g_n = \partial_\theta g_n = \langle(\partial_\theta\phi_n)^\top\xi,\bpsi_n\rangle+
\langle\phi_n^\top\xi,\partial_\theta\bpsi_n\rangle$, $g_n' = \partial_ug_n$, and $\rho_n(\xi)\propto e^{g_n(\cdot;\xi)}$
the row's density. We measure $\xi$ by the sup norms of these fields,
\[
F_n(\xi) = \sup_u\lvert g_n(u;\xi)\rvert+\sup_u\lvert\dot g_n(u;\xi)\rvert ,
\]
and use $\sup_u\lVert\bpsi_n(u)\rVert\le s$ together with the spectral sums $s_1^2 = \sup_u\lVert\partial_u\bpsi_n
\rVert^2 = \sum_k\pi^2k^2A_k^2d_{nk}^2$, $s_\theta^2 = \sup_u\lVert\partial_\theta\bpsi_n\rVert^2$ and
$s_{\theta,1}^2 = \sup_u\lVert\partial_u\partial_\theta\bpsi_n\rVert^2$. For the Mat\'ern spectrum
with $\nu>1$ (and the squared exponential), $\sum_kk^2A_k^2<\infty$, so all four are bounded
uniformly in $K$; the floor of Appendix~\ref{app:spectral} adds $\pi^2\varepsilon
\sum_{k<K}k^2/\sigma^2\approx\pi^2\varepsilon K^3/(3\sigma^2)$ to $s_1^2$ ($3.2/\sigma^2$ at $K = 100$),
and the corresponding term to $s_{\theta,1}^2$ when $\theta$ is the variance, which we keep explicit
as for $s^2$. Since $\lVert\phi_n^\top\Delta\rVert\le\lVert\phi_n\rVert
\lVert\Delta\rVert_F$, Cauchy--Schwarz gives, for all $a,b$,
\begin{equation}
\sup_u\lvert g_n(u;a)-g_n(u;b)\rvert\le s\lVert\phi_n\rVert\lVert a-b\rVert,\qquad
\sup_u\lvert\dot g_n(u;a)-\dot g_n(u;b)\rvert\le\kappa_n\lVert a-b\rVert,
\label{eq:fieldlip}
\end{equation}
with $\kappa_n = s\lVert\partial_\theta\phi_n\rVert+s_\theta\lVert\phi_n\rVert$, and
$\lvert g_n'(v;a)-g_n'(v;b)\rvert\le s_1\lVert\phi_n\rVert\lVert a-b\rVert$. Unlike the
coefficient vectors $\phi_n^\top\xi\in\R^{2(K-1)}$, the fields thus change by amounts free of $K$.

\begin{lemma}\label{lem:growth}
With $c_n = \max\{2\kappa_n+s_1\lVert\phi_n\rVert\lvert\partial_\theta v_n\rvert,\,2s\lVert\phi_n\rVert\}$,
for all $a,b$,
\[
\lvert\partial_\theta\ell_n(a)-\partial_\theta\ell_n(b)\rvert\le c_n\bigl(1+F_n(a)+F_n(b)\bigr)\lVert a-b\rVert .
\]
\end{lemma}

\begin{proof}
In \eqref{eq:dtheta} the first two terms change by at most $\kappa_n\lVert a-b\rVert$ and
$s_1\lVert\phi_n\rVert\lvert\partial_\theta v_n\rvert\lVert a-b\rVert$. For the third,
$\E_{\rho_n(a)}\dot g_n(a)-\E_{\rho_n(b)}\dot g_n(b) = \E_{\rho_n(a)}[\dot g_n(a)-\dot g_n(b)]
+(\E_{\rho_n(a)}-\E_{\rho_n(b)})\dot g_n(b)$; the first part is at most $\kappa_n\lVert a-b\rVert$.
Along $\rho_t\propto e^{g_n(b)+t(g_n(a)-g_n(b))}$, $\frac{d}{dt}\E_{\rho_t}h = \Cov_{\rho_t}(h,g_n(a)-g_n(b))$,
and $\lvert\Cov(h,w)\rvert\le2\sup\lvert h\rvert\sup\lvert w\rvert$, so the second part is at most
$2\sup_u\lvert\dot g_n(u;b)\rvert\,s\lVert\phi_n\rVert\lVert a-b\rVert\le2s\lVert\phi_n\rVert F_n(b)\lVert a-b\rVert$
by \eqref{eq:fieldlip}.
\end{proof}

\begin{lemma}\label{lem:supmom}
Let $X\sim\pi$, $\bar\xi = \E_\pi X$ and $\lambda_n = \lVert\phi_n\rVert(s+s_1+s_\theta+s_{\theta,1})
+\lVert\partial_\theta\phi_n\rVert(s+s_1)$. There is an absolute constant $C$ with
$(\E F_n(X)^2)^{1/2}\le F_n(\bar\xi)+C\lambda_n$.
\end{lemma}

\begin{proof}
The fields are linear in $\xi$, so $F_n(X)\le F_n(\bar\xi)+\sup_u\lvert Z(u)\rvert+\sup_u\lvert\dot Z(u)\rvert$
with $Z(u) = g_n(u;X-\bar\xi)$ and $\dot Z(u) = \dot g_n(u;X-\bar\xi)$, both centred. Since
$\nabla^2U\succeq I$, $\pi$ satisfies a logarithmic Sobolev inequality with constant one
(Bakry--\'Emery), so every $L$-Lipschitz function of $X$ is sub-Gaussian about its mean with
variance proxy $L^2$. $Z(u)-Z(u')$ is linear in $X$ with Lipschitz constant
$\lVert\phi_n\rVert\lVert\bpsi_n(u)-\bpsi_n(u')\rVert\le\lVert\phi_n\rVert s_1\lvert u-u'\rvert$, so $Z$
has sub-Gaussian increments for the metric $\lVert\phi_n\rVert s_1\lvert u-u'\rvert$ on $[0,1]$, whose
covering numbers are at most $1+\lVert\phi_n\rVert s_1/(2\epsilon)$. Dudley's entropy bound, in its
$L^2$ form, gives $(\E\sup_{u,u'}\lvert Z(u)-Z(u')\rvert^2)^{1/2}\le C\lVert\phi_n\rVert s_1$, and
$Z(0)$ is sub-Gaussian with proxy $(\lVert\phi_n\rVert s)^2$. Likewise for $\dot Z$, with increments
of Lipschitz constant $(\lVert\partial_\theta\phi_n\rVert s_1+\lVert\phi_n\rVert s_{\theta,1})\lvert u-u'\rvert$ and
$\dot Z(0)$ of proxy $(\lVert\partial_\theta\phi_n\rVert s+\lVert\phi_n\rVert s_\theta)^2$.
\end{proof}

\begin{proposition}\label{prop:bias}
For any coupling $(X,Y)$ of $\pi$ and $\mu$ attaining $W_2(\mu,\pi)$,
\[
\bigl\lvert\E_\mu\partial_\theta\log p-\E_\pi\partial_\theta\log p\bigr\rvert\le
\sqrt3\sum_nc_n\bigl(1+\E F_n(X)^2+\E F_n(Y)^2\bigr)^{1/2}W_2(\mu,\pi),
\]
with $\E F_n(X)^2\le(F_n(\bar\xi)+C\lambda_n)^2$ by Lemma~\ref{lem:supmom} and
$\E F_n(Y)^2\le2\E F_n(X)^2+2(s\lVert\phi_n\rVert+\kappa_n)^2W_2(\mu,\pi)^2$.
\end{proposition}

\begin{proof}
By Lemma~\ref{lem:growth} and Cauchy--Schwarz, $\E\lvert\partial_\theta\ell_n(X)-\partial_\theta
\ell_n(Y)\rvert\le c_n(\E(1+F_n(X)+F_n(Y))^2)^{1/2}(\E\lVert X-Y\rVert^2)^{1/2}$, and
$(1+p+q)^2\le3(1+p^2+q^2)$. The bound for $Y$ follows from \eqref{eq:fieldlip}:
$F_n(Y)\le F_n(X)+(s\lVert\phi_n\rVert+\kappa_n)\lVert X-Y\rVert$.
\end{proof}

With $\mu = \mu_t$, the law after $t$ sweeps from the previous step's states,
Theorem~\ref{thm:sms} bounds $W_2(\mu_t,\pi)$ by a term that decays geometrically in $t$
plus $C_1h^2(\sqrt{d}+\sqrt{m_R}\,\bar\sigma_*/M_R)$. The factors $c_n$ and $\lambda_n$ in
Proposition~\ref{prop:bias} depend on $J$ and $K$ only through $\lVert\phi_n\rVert$,
$\lVert\partial_\theta\phi_n\rVert$, $\lvert\partial_\theta v_n\rvert$ and the spectral sums, the last
free of $K$ for $\sum_kk^2A_k^2<\infty$ (up to the floor terms above). We have $\lVert\phi_n\rVert\le\sigma(x_n)$,
but we do not bound the feature derivatives or $F_n(\bar\xi)$, the sup norm of the
posterior-mean field and its $\theta$-derivative at $x_n$, uniformly in $J$ and $K$. Under the
theorem's conditions the bias of the Fisher gradient is therefore
$O(h^2(\sqrt d+\sqrt{m_R}\,\bar\sigma_*/M_R))$ plus a mixing term, for a fixed representation, with a
prefactor that is explicit in these quantities. The $\sqrt d$ comes from the Wasserstein bound
itself. The gradient batch is drawn independently
of the sampled states, so it adds variance but no bias. The four chains form two antithetic
pairs and share the sampler's batch partition, so they are not independent replicates.
They reduce variance but should not be treated as independent for Monte Carlo error
estimates, and our validation runs use independent chains for that purpose
(Appendix~\ref{app:validation}). All of this concerns one outer step at fixed
$\vartheta$. We make no convergence claim for the outer procedure, which follows these
gradients by stochastic approximation with a changing target, preconditioner rebuilds and
early stopping.

The Laplace approximation replaces $\log p(y\mid\vartheta)$ by $\mathcal
L(\vartheta) = \log p(y,\hat\xi_\vartheta\mid\vartheta)-\tfrac12\log\det H_\vartheta(\hat\xi_\vartheta)$
with $\hat\xi_\vartheta$ the posterior mode and $H_\vartheta$ the Hessian of $U$ there.
Its gradient is exact for $\mathcal L$, but $\log p(y\mid\vartheta)-\mathcal L(\vartheta)$
depends on the skewness and higher cumulants of the latent posterior, which change with
$\vartheta$, so $\nabla\mathcal L$ is biased by an amount that more computation does not
remove. Computing it requires $\log\det H$ of a $d$-dimensional matrix and the third
derivative of $U$ contracted with $d\hat\xi_\vartheta/d\vartheta$. Variational
approximations likewise optimise a bound whose gap depends on $\vartheta$. Whether this
bias matters in practice is an empirical question, which Appendix~\ref{app:validation}
addresses on a controlled problem.

\section{THE NON-STATIONARY MODEL IN DETAIL}
\label{app:nonstat}

\paragraph{Heads and bounds.}
All encoder outputs pass through a residual scale of $0.01$. The log-amplitude is clipped
to $\pm6$, the response length-scale log multiplier to $\pm3$, the local log length
scales are $3\tanh(\cdot/3)$, and the response-warp coefficients are $\beta_j = (\bar\beta/j)
\tanh(j\tilde\beta_j/\bar\beta)$, with $\tilde\beta_j$ the network output and $\bar\beta=2$. The full-matrix bandwidth $(s,B)$ starts at zero and is
trained by Adam rather than Muon, since it is not a weight matrix of the network. The
input length scales $\ell_x = (\sum_d\Var\,x_d)^{1/2}$ are frozen, and the bandwidth carries
the scale. The weight matrices carry a $\N(0,100)$ prior.

\paragraph{Response warp.}
$e^{r_x}$ is represented as a piecewise-linear function on $256$ segments, so $w_x$ is
its exact (piecewise-quadratic) integral and $w_x'$ is consistent with it. The normaliser
is computed in $v$ on the unchanged quadrature nodes, so the warp enters only through
$v_n$ and the Jacobian.

\paragraph{Learnable transform.}
$T'$ is the normalised piecewise-linear interpolant of $\exp\eta$ on $400$ equispaced
knots spanning the training range plus six kernel bandwidths on each side (starting at $0$ for a
redshift), so that $T$ is exact and quadratic within a cell. $\eta$ is initialised at a
Gaussian kernel density estimate of the training targets with Silverman's bandwidth,
reflected at $y=0$, and carries a Gaussian prior on its second differences whose standard
deviation is the root-mean-square second difference of the initialisation.

\paragraph{Response support.}
Both transforms are tabulated on a bounded interval: the range of the training responses widened by six kernel
bandwidths on each side or, for a redshift, from $0$ (where the kernel estimate is reflected) to the largest training
redshift plus six bandwidths. The fixed transform uses $4{,}000$ points of this interval and the learned one its $400$
knots, so the model's density is supported on the interval. The implementation does not extend it. A response outside
the interval is mapped to $u = 10^{-6}$ or $1-10^{-6}$, and its $\log T'$ is the value at the nearest end of the grid,
so the score it receives is not that of a normalised density. The model itself assigns zero density outside
its support, and its CDF is zero below the lower end and one above the upper end; the clipping convention is a
numerical diagnostic and supplies no predictive tail mass. No test response falls outside the interval in
Table~\ref{tab:main}, on the sinusoid, in the DESI training-size runs of Appendix~\ref{app:validation} or in the \textsc{Happy}
scaling runs of Appendix~\ref{app:slgp}; for the temperature runs of Appendix~\ref{app:slgp} this was not recorded; the covariate-shift
runs of Appendix~\ref{app:validation} have a few such galaxies, listed there.

\paragraph{Preconditioner with damping.}
With per-object damping, $S$ is formed with the amplitudes multiplied by the
root-mean-square damping of each mode over the rows. On SDSS and DESI the preconditioner is
rebuilt at every one of the first ten steps, where the posterior moves fastest, and
every ten steps afterwards. The chains keep their latent states and redraw their velocities.

\section{DATA, PREPROCESSING AND METRICS}
\label{app:data}

\paragraph{Benchmarks.}
\textsc{Happy} and \textsc{Teddy} \citep{beck2017teddy} each comprise samples A to D of
$60{,}000$ to $98{,}000$ SDSS galaxies (about $75{,}000$ in A and B). A is the training set, B is drawn from the same
distribution, and C and D move away from it. One \textsc{Teddy}-B galaxy with two
spectra appears twice and is matched to its photometry by identifier and redshift. The
SDSS DR18 sample of \citet{izbicki2026benchmark} has $500{,}000$ galaxies. The DESI
benchmark comprises the $4.35$ million galaxies of the DR1 Bright Galaxy Survey bright
sub-sample ($r<19.5$, BGS\_BRIGHT, \texttt{ZWARN}${}=0$, $\Delta\chi^2>40$, spectral type
galaxy) with positive fluxes, with the Legacy Surveys $g,r,z$ and WISE $W1,W2$ photometry
of the redshift catalogue \citep{desi2025dr1}. SDSS and DESI are split at random into $90\%$ training and
$10\%$ test rows, a different split for each seed.

\paragraph{Preprocessing.}
The five features (one magnitude, four colours) are standardised by the training rows and
each is multiplied by its median gap, the absolute difference of the median of the
kernel-density CDF of the response over the training rows above and below the feature's
median. The features are an exact linear map $M$ of the magnitudes, so the per-object
feature error standard deviations are $(\sum_bM_{ba}^2\sigma_{nb}^2)^{1/2}$. These five
columns are log-transformed, standardised and median-gap scaled in the same way. Finally
every column is soft-clipped in its own standard-deviation units, the identity up to $3$
and $3+2\tanh((|z|-3)/2)$ beyond, so that objects with failed photometry ($4\%$ of DESI
training rows have some column beyond three standard deviations) do not dominate the
kernel distances.

\paragraph{Metrics.}
NLL of the true response; CDE loss $\E_x\int\hat p(y\mid x)^2dy-2\E\hat p(y\mid x)$;
Kolmogorov--Smirnov distance of the PIT values from uniform; CRPS in the quantile form
$2\int_0^1(\mathbb 1\{y<\hat q(\tau)\}-\tau)(\hat q(\tau)-y)\,d\tau$ on $2{,}001$ levels,
identical for all models; RMSE of the predictive mean; and from the predictive median
$\hat z$, with $\Delta = (\hat z - z)/(1+z)$, the NMAD $1.4826\,\mathrm{median}|\Delta-
\mathrm{median}\,\Delta|$ and the outlier fraction $|\Delta|>0.15$.

\section{EXPERIMENTAL CONFIGURATION}
\label{app:config}

\paragraph{GP.}
Mat\'ern-$3/2$ kernels on both axes, $K=100$, $J=1{,}000$ $k$-means anchors (twenty Lloyd
iterations on the GPU), jitter $\delta = 10^{-3}$; four chains, friction $2$; Muon at
$0.02$ on the weight matrices and Adam at $0.01$ on all other parameters, cosine decay to
zero over $1{,}000$ steps; the validation log-likelihood is evaluated every ten steps, the
best running mean of five values is kept, training stops after $500$ steps without
improvement, and a divergence of more than one nat restores the best state with fresh
sampler random numbers. Each chain's gradient batch has $\min(N_B,24{,}000)$
observations. The sinusoid fits of Section~\ref{sec:sinus} use RBF kernels, $K=200$, $J=1{,}000$
k-means anchors ($J=360$ at $N=400$), a fixed transform, a single batch with $2$ burn-in
and $4$ collection sweeps per step, and $1{,}000$ cosine steps on all training rows, with no
validation set.

\paragraph{TabPFN-3.5.}
Cached-context mode with an int8 key--value cache. The bar probabilities of the eight
members are mapped onto shared bin borders by differencing a single-precision CDF, so bins
far in the tail can be exactly zero. We floor every bin probability at $2^{-24}/8$ when
evaluating the density, which on \textsc{Happy}, \textsc{Teddy} and SDSS changes the NLL only on SDSS (one test galaxy, $0.014$
nats) and \textsc{Teddy} (two galaxies). On DESI its effect is contained in the stable re-scoring of
Appendix~\ref{app:validation}, which changes the mean NLL by at most $4\cdot10^{-5}$.

\paragraph{TabICL-2.}
The v2 regressor checkpoint of the \texttt{tabicl} package with eight ensemble members and
the same context as TabPFN-3.5 (on DESI the same random $500{,}000$ rows), the context
encoded once with cached row representations. Its $999$ predicted quantiles define a
monotone-spline quantile distribution with exponential tails, from which the NLL, PIT, CRPS
and mean are exact; the CDE loss is integrated on a $400$-point grid.

\paragraph{Flow-Spline.}
The conditional density flow of the benchmark of \citet{izbicki2026benchmark}, ported from its
code. It applies monotone rational-quadratic spline layers with linear tails
\citep{durkan2019neural} to the standardised response, on a Gaussian base, with each layer's
parameters given by a two-hidden-layer ReLU network of $x$, and is trained by Adam with weight decay and early stopping on
$10\%$ of the training rows. The benchmark's random search draws eight configurations of
bins, layers, width, learning rate and weight decay and scores each by three-fold
cross-validation of the CDE loss on a $200$-point grid. We score the flow's own density
exactly (NLL and PIT in closed form; mean, CRPS and $\int p^2$ by quadrature in the base
space). The benchmark's renormalised $200$-point grid gives NLLs within $0.011$ nats of these,
except on DESI ($-2.3075$ against $-2.2682$), where the fitted density has narrow spikes.

\paragraph{CRPS ensemble.}
Five networks, each a ReLU perceptron with $64$ hidden units receiving a $16$-dimensional
Gaussian noise vector after its first layer and returning $100$ weighted samples in training and $500$ at test time, trained
by AdamW on the weighted sample CRPS with early stopping (patience $40$ epochs; on the sinusoid $200$, with batch $32$). The
pooled $5\times500$ test samples become a density by an adaptive Gaussian kernel estimate
\citep{silverman1986density,abramson1982bandwidth}, weighted by the samples' weights. Samples
outside the plausible range of the target (the training range extended by its width on each
side) are given zero weight. Whether any fall there varies between runs. On DESI two of the five runs averaged in
Table~\ref{tab:main} placed $17\%$ of their samples outside this range, at
redshifts up to $10^4$, which lowers the finite-sample training loss without improving the
predictive, while the other three placed none (Appendix~\ref{app:validation}).

\section{SCALE}
\label{app:scale}

On DESI ($N = 3.89$ million after the validation split) no array grows with $N$ beyond the
inputs. Feature rows are rebuilt inside the impulses (once per batch, shared by the four
chains), and the preconditioner statistics and the $30$ power iterations run over row chunks. The
Gibbs kernel is evaluated in chunks of $8{,}192$ rows with rematerialisation in the
backward pass. Each chain's gradient batch is capped at $24{,}000$ rows and its
vector--Jacobian product is chunked, and the predictive of the $454{,}000$ validation and test
rows is formed after the sweeps from the stored latent states, $16{,}384$ rows at a time.
When the feature matrix fits in memory (\textsc{Happy}, \textsc{Teddy}, SDSS) it is built
once per outer step instead. Sharing the batch partition across chains alone made the
large-scale step about $1.7$ times faster. Scoring the $434{,}505$ DESI test galaxies takes
about a minute on the GPU.

\section{SAMPLER VALIDATION AND ABLATIONS}
\label{app:validation}

\paragraph{Fixed-target sampler validation.}
On the sinusoid with $N = 2{,}000$ and a reduced representation ($J = 100$, $K = 50$,
$d = 9{,}800$), we freeze $\vartheta$ at the value learned by sampling (variance $77$,
$\ell_x = 0.147$, $\ell_y = 0.110$). At this $\vartheta$ the whitened curvature estimate is
$\hat\lambda = 9.8$, while the certified constants are $M_R = 5.7\cdot10^4$ (batch-uniform,
$n_b = 4$; $M_R^{\mathrm{full}} = 1.4\cdot10^4$) and $M_R/m_R = 1.3\cdot10^8$. The regime of
Theorem~\ref{thm:sms} ($\gamma\ge\sqrt{8M_R}\approx670$) is therefore far from anything that can be run.
The uniform bound of Theorem~\ref{thm:sms} is $\bar\sigma_* = 2sN\max_n\lVert\phi_n\rVert = 3.3\cdot10^4$,
so its extra term $\sqrt{m_R}\,\bar\sigma_*/M_R\approx0.012$ is negligible next to $\sqrt d = 99$. (For
illustration, the spread $\sigma_*(\mathcal P)$ of $20$ random partitions at the mode is $446$ in $\xi$
coordinates, against $454$ expected from Corollary~\ref{cor:sigma}, and $28$ in whitened coordinates.)

We validate the default sampler in two ways. The first is step-size halving. SMS--UBU with the
default $\gamma = 2$ runs at $h$ (the runner's default, $h = 0.159$), $h/2$ and $h/4$, each for the
same simulated time ($2{,}000$ burn-in and $20{,}000$ collection sweeps at $h$, scaled by $2$ and $4$),
on $16$ independent, non-antithetic chains with the Kronecker whitening fixed at the mode. For every
latent component and for the field $g(u,x)$ on a $16\times32$ grid we compare posterior means
and variances between neighbouring step sizes, in units of the posterior standard deviation,
with Monte Carlo errors from the spread of the chains. The second is a reference that leaves
$\pi$ invariant, generalised HMC (GHMC) with one BCSS-2 step per Metropolis decision and partial
momentum refreshment \citep{chok2025twosystem}, without adaptation, in the same whitened
coordinates and with full-batch gradients. Its step is fixed before sampling by the rule of
\citet{chok2025twosystem} (acceptance at least $1-h_{\max}/16$). We also restart both samplers
from dispersed initial states (independent prior draws), report $\hat R$ and effective sample
sizes for every compared functional, and repeat the comparison on a problem small enough
($J = K = 10$, $d = 180$) for the reference to be run to high precision.

At the default step the latent means show no detectable bias and the variances a small one.
Between $h$ and $h/2$ the posterior means of the $9{,}800$ latent components differ by $0.006$
posterior standard deviations (root mean square), exactly their Monte Carlo error ($5.3\%$ of
components beyond two standard errors, against $4.6\%$ expected). The variances differ by $0.6\%$,
of which $0.3\%$ is resolved above the noise, and between $h/2$ and $h/4$ the corresponding biases
are $0.002$ standard deviations and $0.1\%$. The Fisher gradient, a sum over all $2{,}000$ rows, is
more sensitive. Between $h$ and $h/2$ its components (for $\log\sigma^2$, $\log\ell_x$, $\log\ell_y$)
change by $-3.3$, $+6.2$ and $+6.0$ ($z = -10$, $2.9$ and $7.8$), and between $h/2$ and $h/4$ by
$+0.2$, $-3.9$ and $+0.6$ ($z = 0.6$, $-2.3$ and $0.9$). Its bias at $h$ is thus resolved and
shrinks with the step at a rate consistent with $O(h^2)$ or faster. (A ratio of four between
successive differences would also require a non-zero leading coefficient, which this resolution
cannot establish.) The effect on learning is small. With the sampler run at $h/2$ for the same
simulated time throughout the optimisation, the learned $\sigma^2$ is $8\%$ higher on average (less
than its spread over seeds), the length scales change by less than $2\%$, and the test NLL is
unchanged ($-0.1918$ against $-0.1920$, per-seed differences within $0.003$). For the field on the
grid, the $h/2$ versus $h/4$ differences of the posterior means exceed the chain-spread error
($14\%$ of grid points beyond two standard errors). The retained-draw diagnostics below show that
this error estimate cannot be relied on for the field, and they supersede these comparisons.

The reference confirms the findings for the latent components and the gradient. GHMC ($32$ chains, half started at the mode and
half at prior draws, each $240{,}000$ steps after $24{,}000$ of burn-in; $h_{\max} = 0.17$ and
$0.26$, acceptance $0.996$ and $0.985$) agrees with SMS--UBU at the default $h$ on the latent
means (root-mean-square difference $0.005$ posterior standard deviations, equal to the Monte
Carlo error) and on the latent variances to $0.4\%$. The Fisher
gradient at $h$ differs from the reference by $-3.5$, $+6.1$ and $+7.5$ ($z = -13$, $3.1$ and
$10$), much as between $h$ and $h/2$, and at $h/2$ by $-0.3$, $-0.1$ and $+1.5$ ($z = -1.1$, $-0.1$
and $2.6$). The discretisation bias of the gradient is therefore real, and halving the step
removes most of it, although one component remains at $z = 2.6$. Learning at $h/2$ changes the
fitted variance by $8\%$ and the test NLL by $0.0002$ nats (see above). Started from independent
prior draws, both samplers agree with their mode-started runs within Monte Carlo error (gradient
$z$ at most $1.2$). At $h/4$, however, the field variances differ from the reference beyond the
chain-spread error.

To look into this we repeated all five runs (SMS--UBU at $h$, $h/2$ and $h/4$, and GHMC from the
mode and from prior draws, with $16$ chains each), this time keeping $2{,}000$ draws per chain. From
these we computed rank-normalised, folded split-$\hat R$, bulk and tail effective sample sizes and
ESS-based Monte Carlo errors \citep{vehtari2021rhat} for the Fisher gradient, $256$ randomly chosen
latent components, and the field and the normalised density on the grid. The gradient and the
latent components are well resolved in every run ($\hat R\le1.006$, bulk ESS at least $3{,}200$,
tail ESS at least $8{,}800$). The field and density values are not. In every run, the references
included, $\hat R$ exceeds $1.01$ at $40$--$85\%$ of the grid points (up to $1.13$), and bulk ESS
falls to about $100$. These diagnostics do not resolve the field and normalised-density functionals well
enough to distinguish discretisation bias from finite-run error, so we exclude them from our
claims of validated posterior accuracy; this includes the $h/4$ discrepancy above. With
ESS-based errors the two GHMC runs agree with each other (gradient $\lvert z\rvert\le0.6$). SMS--UBU
at $h$ differs from them in the variance component of the gradient ($z = -12$), and at $h/2$ and
$h/4$ by at most $\lvert z\rvert = 2.2$. For every step size, the fraction of the $256$ latent means
and variances that lie more than two errors from the reference ($3$--$6\%$) is what chance would
give, as it is for the two GHMC runs ($2\%$ and $8\%$).

On the small problem ($J = K = 10$, $d = 180$, $\vartheta$ refitted by sampling), where the
reference is precise, SMS--UBU at $h$ shows a resolved bias of $0.6\%$ in the latent variances and
of $-0.86$ ($z = -4.7$) in the variance component of the gradient, and none at $h/2$ or $h/4$. To
summarise, the default step has small but detectable biases in the latent variances and the
gradient. They shrink substantially at half the step and have little observed effect on the
predictive NLL of the controlled problem. For the field and normalised-density values on the grid,
these runs, the reference included, cannot tell discretisation bias from finite-run error.

\paragraph{Same-representation MAP and Laplace.}
On the same representation, $\vartheta$ is learned from the same starting value by three
methods: our sampler with Fisher-identity gradients, the latent mode with plug-in gradients
(MAP), and the Laplace approximation of the marginal likelihood with an exact log-determinant
and the implicit derivative of the mode. Each is scored on the $2{,}000$ test rows by its own
predictive (posterior-mean density; plug-in density at the mode; Gaussian around the mode)
and, to separate learning from prediction, by the sampled predictive at its learned
$\vartheta$ (Table~\ref{tab:small}). All three start from $\log\vartheta_0 = (1.6, -1.2,
-1.2)$ and take $300$ Adam steps with learning rate $0.03$ and a flat prior on
$\log\vartheta$. MAP maximises the joint density $p(y,\xi\mid\vartheta)$ in the non-centred
parametrisation, with the mode refreshed by five Newton steps per outer step. The Laplace
predictive averages the normalised densities of $200$ independent draws from
$\N(\hat\xi,H^{-1})$, and the sampled predictive those of $200$ states of four fresh SMS--UBU
chains, so the two entries of the sampling row are separate Monte Carlo estimates. MAP
collapses on every seed to a large variance and length scales five to twelve times shorter than
the others, which overfits the field. This reflects joint maximisation in this
parametrisation and says nothing against deterministic latent fits in general. Laplace learns
hyperparameters close to those of sampling, and at its $\vartheta$ the sampled predictive is
as good as ours, but its own Gaussian predictive is $0.12$ nats worse. This is an observed finite-run comparison
under the reported budgets. Under these budgets the larger observed difference lies in the predictive
evaluations rather than in the learned hyperparameters. The comparison does not separate the error of the
Gaussian approximation from the remaining finite-run error of the sampled predictive.

\begin{table}[h]
\caption{Learning $\vartheta$ on the sinusoid ($N=2{,}000$, $J=100$, $K=50$): mean (standard
deviation) of the test NLL over four seeds, by each method's own predictive and by the
sampled predictive at its learned $\vartheta$, and the learned variance and length scales. The
seed standard deviations are neither a convergence certificate nor an estimate of the predictive's
Monte Carlo error.}
\label{tab:small}
\begin{center}
\footnotesize\setlength{\tabcolsep}{3pt}
\begin{tabular}{lcccc}
\toprule
& own predictive & sampled pred. & $\sigma^2$ & $(\ell_x,\ell_y)$ \\
\midrule
sampling & $-0.192\ (0.014)$ & $-0.192\ (0.016)$ & $59$--$78$ & $(0.13$--$0.15,\ 0.11$--$0.12)$ \\
MAP & $+1.577\ (0.225)$ & $+1.308\ (0.166)$ & $3983$--$5062$ & $(0.027,\ 0.010)$ \\
Laplace & $-0.076\ (0.013)$ & $-0.192\ (0.016)$ & $75$--$94$ & $(0.14$--$0.15,\ 0.11$--$0.12)$ \\
\bottomrule
\end{tabular}
\end{center}
\end{table}

\paragraph{Transform-only baseline.}
With $\xi\equiv0$ the model is $p_0(y\mid x) = T'(y)\,w_x'(T(y))$, a conditional density
estimator in its own right. Trained with the same encoder, transform, schedule and
validation split as the full model (seed $0$), it reaches test NLLs of $-1.788$
(\textsc{Happy}), $-2.195$ (\textsc{Teddy}), $-1.763$ (SDSS) and $-2.190$ (DESI), $0.13$ to
$0.37$ nats behind the full model, with PIT statistics of $0.05$--$0.10$ against $0.001$--$0.005$.
Removing the latent field thus materially degrades both density accuracy and calibration in
this configuration (one seed per benchmark). Since the ablation removes capacity as well as
posterior integration, it shows the value of the field but not, by itself, the value of sampling.

\paragraph{Matched data on DESI.}
Trained on the same $500{,}000$ rows as TabPFN-3.5's context (validation rows removed from
them), the GP with the $10\times128$ network overfits (validation NLL $-2.392$, test
$-2.352$), $0.005$ nats behind TabPFN-3.5 ($-2.357$). With a $10\times64$ network it reaches
$-2.360$, $0.003$ nats ahead (paired bootstrap interval $[0.002, 0.004]$), with a PIT
statistic of $0.004$ against $0.006$. With the $10\times128$ network the test NLL is
$-2.304$, $-2.352$, $-2.367$ and $-2.379$ at $10^5$, $5\cdot10^5$, $10^6$ and
$3.9\cdot10^6$ rows. Between the two networks the validation rows also favour the smaller one
($-2.394$ against $-2.392$), so the choice does not rest on the test rows. The bootstrap interval
is conditional on the selected fit and includes neither selection nor training randomness, so
this $0.003$-nat lead is less certain than the full-data difference below. At equal data the two
models are therefore close, and the GP's lead on the full benchmark comes mostly from using
all of it.

\paragraph{Repeated runs and paired differences.}
Every GP entry of Tables~\ref{tab:sinus} and~\ref{tab:main}, and every TabPFN-3.5 entry except on \textsc{Happy} and
\textsc{Teddy} (where it does not depend on the seed), is repeated over four seeds (the validation split on \textsc{Happy}
and \textsc{Teddy}, the train/test and validation splits on SDSS and DESI, the train/test split on the sinusoid). With the test rows fixed, the seed moves the GP's test NLL by $0.0009$ nats
(standard deviation) on \textsc{Happy} and \textsc{Teddy}. With new test rows it moves it by
$0.0028$ on SDSS, $0.0010$ on DESI and by $0.015$--$0.039$ on the $2{,}000$ sinusoid test rows,
where the excess over the true density on the same rows varies far less ($0.001$ at $10{,}000$
training rows). The paired GP minus TabPFN-3.5 NLL difference varies less than either model:
$-0.0111$ (standard deviation $0.0007$) on SDSS, $-0.0222$ ($0.0004$) on DESI, and $-0.017$
($0.016$), $-0.009$ ($0.013$) and $-0.016$ ($0.004$) on the sinusoid at $400$, $2{,}000$ and
$10{,}000$ rows. Within the seed-$0$ split, paired bootstrap intervals over test objects are
$[+0.006, +0.009]$ on \textsc{Happy}, $[+0.005, +0.007]$ on \textsc{Teddy}, $[-0.013, -0.009]$
on SDSS and $[-0.023, -0.022]$ on DESI, so every difference between the two models in
Table~\ref{tab:main} is resolved. The two kinds of interval measure different things. The seed
standard deviations include the split (where it varies), the validation split and training
randomness, whereas the bootstrap intervals reflect only the finite test sample, conditionally on
the fits. Neither includes the Monte Carlo error of the predictive, which is small. Re-scored from
the same trained chains,
the $20$-sweep production predictive is within $0.00014$ nats (\textsc{Happy}), $0.00002$
nats (SDSS) and $0.00001$ nats (DESI) of a $200$-sweep one, two independent $200$-sweep predictives
differ by $0.00001$, $0.00002$ and $0.00001$, and the PIT statistic does not change. The production predictive's per-object
Monte Carlo error is about $0.003$, $0.0013$ and $0.0006$ nats in $\log p$ (\textsc{Happy}, SDSS,
DESI), against $0.0012$, $0.0005$ and $0.0003$ at $200$ sweeps.

\paragraph{Evaluation audit.}
The TabPFN-3.5 floor changes the aggregate SDSS test NLL (seed-$0$ split) from $-2.0862$ to $-2.0998$, in
TabPFN's favour. The floor is applied to the bar probabilities before taking logarithms,
without renormalisation, and sits at one single-precision rounding step of the members' CDF. The CRPS
ensemble's members diverge in some DESI runs and not in others. Our first run ($-2.2564$)
placed $17\%$ of its samples outside the plausible range
(Appendix~\ref{app:config}), where they receive zero weight. A second run with identical
settings placed none there and scored $-2.2855$, or $-2.2848$ without the range restriction, so
the restriction itself does not matter when no sample leaves the range. Of five runs with
identical settings, two diverge ($17\%$ of samples outside the range; $-2.2564$ and $-2.2563$)
and three do not ($-2.2855$, $-2.2819$ and $-2.2791$, mean $-2.282$, standard deviation
$0.003$). Divergence costs about $0.03$ nats, and even the runs that do not diverge remain at
least $0.09$ nats behind the LGP. Table~\ref{tab:main} reports the mean and standard deviation
over all five runs, so this baseline's DESI entry carries a between-run spread of about
$0.014$ nats that the other DESI entries do not show. Renormalising TabPFN-3.5's floored bar probabilities changes the
per-object NLL by $\log S = \log\sum_b\max(p_b,\epsilon)$, at most $2.8\cdot10^{-5}$ on every
benchmark, and the renormalised mean NLL differs from the reported one by $1.5\cdot10^{-5}$
(\textsc{Happy}), $1.0\cdot10^{-5}$ (\textsc{Teddy}) and $1.6\cdot10^{-5}$ (SDSS). The float32
differencing of TabPFN-3.5's CDF limits how finely tail probabilities are resolved. The $13$, $11$ and
$8$ test galaxies whose bar probability at the observation is within $100$ such steps of zero
contribute $0.0014$, $0.0012$ and $0.0014$ nats to the mean NLL. Since this contribution is not an error bound, we recomputed for every test galaxy the probability of the shared bar
that holds its observation, starting from each ensemble member's own logits and working in double precision and in log space.
Instead of differencing a single-precision CDF, this sums the exact overlaps of the member's bars with the shared bar; the other
bars, which set the softmax normaliser, are left as they were. The stable score changes the mean test NLL by at most
$4\cdot10^{-5}$ nats on every benchmark and split (\textsc{Happy}, \textsc{Teddy}, four SDSS and four DESI splits). Single galaxies
at the resolution limit move by up to $1.9$ nats in either direction, but there are few of them ($2$--$13$ per split on \textsc{Happy}, \textsc{Teddy} and SDSS, $155$--$226$
of $434{,}505$ on DESI), and every paired GP minus TabPFN-3.5
difference of Section~\ref{sec:experiments} changes by less than $10^{-4}$. Re-tuned with the exact CDE loss of its density in place of the
$200$-point grid, Flow-Spline selects the same configuration on SDSS and its scores are
unchanged. We did not complete the same check on DESI ($80$ minutes per search configuration),
which remains a limitation of that entry.

\paragraph{Conditional calibration.}
Within deciles of the magnitude column, the GP's PIT statistic stays at most $0.017$
(\textsc{Happy}), $0.018$ (\textsc{Teddy}), $0.021$ (SDSS) and $0.007$ (DESI), and the
coverage of its central $90\%$ intervals between $0.890$ and $0.906$, consistent with the
sampling noise of a decile. TabPFN-3.5 reaches $0.027$, $0.025$, $0.047$ and $0.024$. Its
worst decile is the brightest one, where its mean PIT is $0.51$--$0.53$, so it underpredicts
the redshifts of bright galaxies. Calibration within the bins of one covariate is only a
group-conditional check, not calibration given the full covariate vector. Predictive
coverage also does not validate the posterior uncertainty about the density function itself.

\paragraph{Covariate shift.}
Trained on sample A and tested on the more distant samples C and D, the GP meets two limitations.
First, every shifted test set contains redshifts above the support of its learned transform
(Appendix~\ref{app:nonstat}): $7$ of $60{,}315$ and $26$ of $74{,}642$ galaxies on \textsc{Happy} C and D, and $1$ of
$97{,}980$ and $35$ of $75{,}924$ on \textsc{Teddy} C and D. The model gives them zero density, so its full-test NLL
on these four sets is infinite. The scorer's clipping convention gives them finite but poor scores ($23$--$25$ nats,
against $4$--$10$ for TabPFN-3.5), which add $0.0002$ to $0.011$ nats to the GP's mean score. We therefore compare
log-likelihoods on the galaxies inside the support, the same rows for both models, computed from the saved
per-galaxy scores. This subset is defined by the observed redshift and excludes galaxies on which the GP scores
badly, so it favours the GP. Second, the GP falls further behind TabPFN-3.5 as the shift grows. On the in-support
galaxies it is $0.021$ and $0.033$ nats behind on \textsc{Happy} A$\to$C and A$\to$D (paired bootstrap $[0.018, 0.024]$
and $[0.030, 0.036]$; the clipped means give $0.022$ and $0.038$), against $0.008$ on A$\to$B, while its PIT is
slightly closer to uniform ($0.0085$ and $0.019$ against $0.0096$ and $0.022$). For a galaxy above the support the
model's PIT is one, and the clipping convention returns at most $7\cdot10^{-5}$ less, which leaves these statistics
unchanged at the reported precision. On \textsc{Teddy} A$\to$C the gap is $0.008$, as on A$\to$B. On \textsc{Teddy}
A$\to$D, where $95\%$ of the test galaxies have a magnitude or colour outside the central
$99.8\%$ of the training range (against $1$--$2\%$ for B and C, and $11\%$ for \textsc{Happy}
D), the GP fails: on the in-support galaxies its NLL is $-0.33$ against $-1.66$ for TabPFN-3.5 (clipped means $-0.32$
and $-1.66$), with PIT statistics of $0.22$ and $0.12$. Far from the
anchors the Nystr\"om features vanish, so the latent field reverts to zero and the density to
the learned transforms (the transform-only baseline above), whose predictions barely depend on
the galaxy. To check this mechanism we measured the magnitude of the latent field $\lVert P_n\rVert$, averaged over
posterior states. On the in-distribution test samples (\textsc{Happy} B, \textsc{Teddy} B, SDSS, DESI)
its distribution matches that on the training galaxies (medians within $3\%$; $0.1\%$ on DESI), while on
\textsc{Teddy} D its median is half the training median and its $5\%$ quantile a fifth of the
training one. So the mechanism is at work, but it does not explain everything. On
\textsc{Happy} D the field shrinks almost as much (median about half the training median), yet
the LGP is only $0.033$ nats behind TabPFN-3.5 there on the in-support galaxies. Extrapolation beyond the training
covariates is therefore a limitation of the present covariate representation, as is the response support fixed by
the training range, and our claims mainly concern in-distribution conditional density estimation.

\paragraph{Representation checks.}
On \textsc{Happy} A$\to$B the test NLL is $-2.1543$ and $-2.1555$ with $K = 50$ and $200$
modes, $-2.1558$ with $200$ quadrature cells, and $-2.1543$ and $-2.1557$ with $J = 500$ and
$2{,}000$ anchors, against $-2.1557$ (standard deviation $0.0009$ over seeds) for the defaults
($K = 100$, $100$ cells, $J = 1{,}000$). Halving any of them costs at most $0.0014$ nats, and
doubling gains nothing. On the small sinusoid model ($K = 50$), the log-normaliser of sampled
states agrees with a rule four times finer to $7\cdot10^{-6}$ at the learned $\ell_y$, but
only to $0.47$ at the shortest damped length scale $e^{-3}\ell_y$, so the permitted range of
the damping includes configurations the rule does not resolve. We check at fixed posterior
states whether fitted models reach them. For $50{,}000$ test objects and $16$ posterior
states of the fitted models we compare $\log Z_Q$ with a rule with four times the cells, and
its gradient with respect to $P_n$. The largest difference in $\log Z$ is $1.1\cdot10^{-5}$ on
\textsc{Happy}, $3.3\cdot10^{-6}$ on \textsc{Teddy} and $4.9\cdot10^{-5}$ on SDSS (99.9\% quantiles
$8\cdot10^{-6}$, $2\cdot10^{-6}$ and $1.1\cdot10^{-5}$), near single-precision rounding, and the
largest relative error of the gradient is $3\cdot10^{-4}$. The difference does not grow across
deciles of the learned damping, whose log multiplier stays within $[-0.56, 0.63]$, far from the bound $\pm3$
of the stress test. On DESI the damping reaches further, to log multipliers of $-1.3$ and $-1.5$ in two fits, and at
fixed states the rule no longer resolves every object. The $99.9\%$ quantile of the difference is $1.1\cdot10^{-5}$, as on SDSS,
but the largest is $3\cdot10^{-3}$--$9\cdot10^{-3}$ (relative error of the gradient with respect to $P_n$ up to $1.1\cdot10^{-2}$),
all in the decile of the shortest length scales. We therefore measured the effect on the score directly. Re-scoring the production predictive
of a DESI fit (its $960$ states, every fifth impulse of the last $15$ of the $20$ sweeps in four chains, all $434{,}505$ test galaxies) with $4$ and $16$ times the cells changes the test NLL by
$1\cdot10^{-6}$ nats and the per-object log-density by at most $1.3\cdot10^{-2}$ ($31$ galaxies above $10^{-3}$, $99.9\%$
quantile $3.4\cdot10^{-6}$), and the two finer rules agree to $1.5\cdot10^{-6}$, so the reference itself is resolved. This is the
scoring error at fixed states. We did not measure how the rule changes the fitted posterior or the outer optimisation, and the
gradient check concerns $P_n$, not the hyperparameters. (Refitting \textsc{Happy} with $200$ cells
moves the test NLL by $0.0001$, but a refit can adapt to the rule, so this is weaker evidence.)

\section{ADDITIONAL RESULTS}
\label{app:extra}

On \textsc{Happy} A$\to$B the $3\times32$ network's test log-likelihood is flat from step
$500$ to the early stop. On \textsc{Teddy} the
validation optimum is at step $310$ and the test log-likelihood falls by $0.005$ nats over
the following $500$ steps, which the early stop prevents. On SDSS (four seeds) the smoothed
validation optimum is at steps $870$--$1{,}000$, and in seed $0$ the test NLL is $-2.1111$, $-2.1110$
and $-2.1112$ at steps $700$, $800$ and $900$. On DESI it is at steps $990$--$1{,}000$ in every
seed, and in seed $0$ the test NLL improves from $-2.3686$ at step $500$ to $-2.3748$, $-2.3779$
and $-2.3791$ at steps $700$, $800$ and $900$. The DESI response warps have a median largest
stretch of $3.33$, $2.89$, $3.71$ and $3.13$ over the test objects in the four seeds (median
smallest $0.61$, $0.60$, $0.57$ and $0.52$), and the learned transform changes the
log-density by at most $2.9$--$3.2$, at $z = 0$.

\section{ILLUSTRATION}
Figure~\ref{fig:sinusoid-density-illustration} illustrates the model described in Section~\ref{sec:sinus}.

\begin{figure*}[t]
\centering
\includegraphics[width=\textwidth]{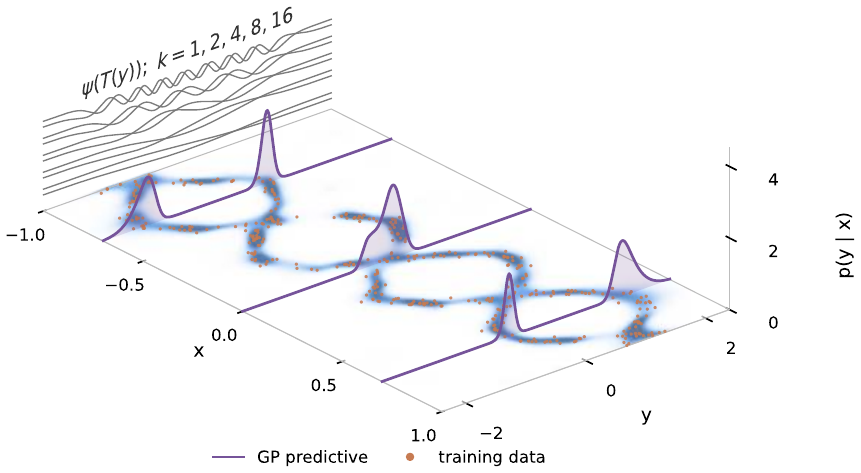}
\caption{The vanilla logistic GP on the bimodal sinusoid with $N=400$
training samples (orange points). The floor heatmap shows the posterior
predictive density; purple profiles show the predictive at $x=-0.7,0,0.7$.
Gray basis curves on the $x=-1$
boundary illustrate $\cos(\pi k T(y))$ and $-\sin(\pi k T(y))$ for
$k=1,2,4,8,16$, evaluated using the fixed training-only KDE-CDF transform
$T$. Each sine/cosine pair shares a display baseline. The basis curves are
scaled and vertically offset for readability, and their heights are not
on the density axis.}
\label{fig:sinusoid-density-illustration}
\end{figure*}

\section{COMPARISON WITH SPATIAL LOGISTIC GAUSSIAN PROCESSES}
\label{app:slgp}

\paragraph{Relation.}
The spatial logistic Gaussian process (SLGP) of \citet{gautier2026continuous} is the model closest to ours. Its conditional
density is the logistic transform \eqref{eq:lgp} of a finite-rank Gaussian process on (covariate, response), built here from
random Fourier features of a stationary kernel on $[0,1]^{D+1}$, with the response confined to a fixed interval and the normaliser
computed by quadrature. The R package \citep{gautier2026slgp} fits the coefficients by MAP, a Laplace approximation or NUTS (Stan),
and the accompanying paper proves that the coefficient likelihood is log-concave \citep[Theorem~3.2]{gautier2026slgp}. With the
Gaussian prior the posterior is therefore strongly log-concave, the property behind our Proposition~\ref{prop:lc}; we claim
neither this property nor latent posterior sampling as new. The two methods differ in four ways. Our representation is a Fourier
basis on the response times Nystr\"om features of the covariates, where SLGP uses random features of the joint kernel. Our
sampler is SMS--UBU in Kronecker-whitened coordinates, theirs is NUTS. We learn the hyperparameters by Fisher-identity gradients of
the marginal likelihood, whereas SLGP searches a grid of length scales on the posterior at the MAP, with the variance fixed by a
heuristic. Finally, our response axis is mapped by a transform estimated from the training responses (fixed or learned;
its support is given in Appendix~\ref{app:nonstat}), theirs is a fixed interval chosen in advance. Because the two
outer objectives differ, they need not select the same hyperparameters even with unlimited computation.
\citet[Section~5]{gautier2026continuous} list quantitative predictive evaluation, comparisons across kernel settings and
baselines, and a more efficient implementation for more observations and higher-dimensional inputs as further work. This section
compares the two methods on held-out data, in their application and in two of ours.

\paragraph{Protocol.}
We run SLGP with the authors' code, CRAN \texttt{SLGP}~2.0.0, and for the saved models of \citet{gautier2026continuous} with the
package build their experiments ran. Both methods see the same training rows and are scored on the same test rows. Test rows play no
part in preprocessing, in the selection of hyperparameters within a search space, in stopping or in the choice of SLGP's
response interval. They did inform the design of the sinusoid search space, as described below. Densities are in the response's original units,
Jacobians included, and the predictive averages densities over posterior draws before taking the logarithm. We score SLGP from its
fitted coefficients. The package's basis functions are evaluated directly, and each draw's normaliser numerically by the trapezoid
rule on a $1{,}401$-point response grid per test input (refined below), because in the configurations we examined the package's own
prediction routine fails off its quadrature grid in 2.0.0 and overstates densities by about $1\%$ in the paper's build (both
checks are in our code repository, \url{https://github.com/paulindani/LGP}). We compare SLGP with the package defaults
(length scale $0.15$ on every coordinate scaled to $[0,1]$, $\sigma^2$ by the package heuristic, $500$ basis functions) and SLGP
tuned on held-out training rows (below), each fitted by MAP, Laplace and NUTS ($2$ chains of $1{,}000$ iterations, the paper's
setting) where the budget allows. Our model is run in its stationary form (the vanilla model of Section~\ref{sec:model} with ARD covariate length
scales, no encoder, no validation set) and in its full form (Section~\ref{sec:nonstat}, with early stopping). In the temperature and \textsc{Happy} comparisons each fit had a declared
wall-clock limit (SLGP MAP and Laplace $1$~h, the SLGP search $1.5$--$2$~h, NUTS $2.5$~h, our model $1$~h). The sinusoid comparison
and the \textsc{Happy} repeat search and checks were run to completion. A run that exceeded its limit is reported as such, which bounds the budget, not the
method. For the temperature and \textsc{Happy} comparisons SLGP ran on two CPU cores per fit (AMD EPYC 7F72) and our model on one
NVIDIA A100 (40\,GB). The sinusoid comparison and the \textsc{Happy} repeat search and checks ran on two Intel Xeon Gold 6530 cores per
SLGP fit and one RTX~5090. The timings therefore compare two systems rather than two algorithms.

SLGP's $\sigma^2$ heuristic sets the prior variance so that the expected prior range of the latent field is about $5$
\citep{gautier2026slgp}. This is a choice of prior, not a constraint on the fitted field. On the sinusoid below (seed $42$,
$N=400$), the best NLL on held-out training rows over the length scales is $+0.47$ with the heuristic $\sigma^2$, $+0.00$ with ten times it and
$-0.13$ with a hundred times it, so the default prior regularises this problem too strongly. We therefore select $\sigma^2$ jointly
with the length scales by held-out NLL on training rows, which targets the accuracy of the predictive density. We hold out $20\%$ of
the training rows, fit the MAP on the rest at every point of a grid (response length scale, one covariate length scale, and
$\sigma^2$ as a multiple of the heuristic; grids below), select the point with the lowest held-out NLL and refit on all training
rows. A selection on the boundary of its grid is repeated on an extended grid. Preliminary fits on seed~$42$ of the
sinusoid (four probe fits at $N=400$), scored on its test rows, informed the decision to include $\sigma^2$ in the search. Within
the resulting search spaces, the final hyperparameters were selected by held-out NLL on training rows. The sinusoid comparison
therefore includes test-informed protocol development and is not an entirely test-independent evaluation, and we regard it as
exploratory. The probe fits used seed~$42$ only.

\paragraph{Swiss temperatures.}
The application of \citet[Section~4.3]{gautier2026continuous} models the daily mean temperatures of 2019 at MeteoSwiss stations
given latitude, longitude and altitude. Its split trains on the $26$ stations outside the canton of Bern ($9{,}490$ days, but only
$26$ distinct inputs) and tests on the $3$ Bern stations ($1{,}095$ days). We take the paper's data and saved models from its
repository. Refitting with its code and package build reproduces the saved model ($\sigma^2$ identical, random frequencies to
$4\cdot10^{-15}$, MAP coefficients to $7\cdot10^{-7}$). Its length scales ($7.5\%$ of the temperature range, $40\%$ of the latitude
and longitude ranges, $15\%$ of the altitude range) were chosen on a $1{,}100$-point grid whose saved objective follows the
training posterior at the MAP (Spearman correlation $0.96$ when we re-evaluate it), not the test stations. We add SLGP with length
scales selected on the same grid by station-grouped six-fold cross-validation on the $26$ training stations ($5\%$ / $50\%$ / $50\%$
/ $25\%$). Our stationary model uses ARD Mat\'ern-$5/2$ kernels on both axes (SLGP's smoothness) on the standardised inputs and is
fitted on all $26$ stations without a validation set, with five seeds. The full model has the SDSS configuration of
Table~\ref{tab:config} ($6\times64$ encoder, $J_w=6$) with early stopping on three held-out training stations, using five splits (each holding
out three different stations) and three seeds. Table~\ref{tab:slgp-meteo} gives the scores and Figure~\ref{fig:slgp-meteo} the
predictive densities.

\begin{table}[h]
\caption{Swiss temperatures, the $3$ Bern test stations ($1{,}095$ days, $^\circ$C). NLL and CRPS: mean over the test days; NLL per
station; $\Delta$: paired NLL difference to SLGP with the paper's length scales (MAP) and to SLGP with station-CV length scales
(MAP), $\pm$ the standard error of a block bootstrap over the $52$ weeks of the year (the days of a week at the three stations form
a block, since daily temperatures are autocorrelated). Rows over several runs: mean over the runs, standard deviation in parentheses.
SLGP: $500$ basis functions, Mat\'ern-$5/2$; paper's $\ell$: the saved fits of \citet{gautier2026continuous} (its
package build); CV $\ell$: fitted with CRAN \texttt{SLGP}~2.0.0; LGP: our model. Bold: the lowest value in each column except the $\Delta$NLL columns.}\label{tab:slgp-meteo}
\begin{center}
\footnotesize\setlength{\tabcolsep}{3pt}
\resizebox{\textwidth}{!}{\begin{tabular}{lcccccccc}
\toprule
& & & & \multicolumn{3}{c}{NLL per station} & \multicolumn{2}{c}{$\Delta$NLL vs SLGP} \\
\cmidrule(lr){5-7}\cmidrule(lr){8-9}
model & NLL & CRPS & RMSE & Bern & Grimsel & Meiringen & paper $\ell$ & CV $\ell$ \\
\midrule
SLGP, paper's $\ell$, MAP & $3.4000$ & $4.251$ & $7.39$ & $3.371$ & $3.361$ & $3.469$ & --- & $+0.016\pm0.009$ \\
SLGP, paper's $\ell$, NUTS & $3.4007$ & $4.251$ & $7.39$ & $3.371$ & $3.361$ & $3.470$ & $+0.001\pm0.001$ & $+0.017\pm0.008$ \\
SLGP, CV $\ell$, MAP & $3.3838$ & $4.216$ & $7.34$ & $3.366$ & $3.355$ & $3.431$ & $-0.016\pm0.009$ & --- \\
SLGP, CV $\ell$, NUTS & $3.3841$ & $4.217$ & $7.34$ & $3.366$ & $3.356$ & $3.431$ & $-0.016\pm0.009$ & $+0.000\pm0.000$ \\
\midrule
LGP, stationary (5 seeds) & $\mathbf{3.3812}$ ($0.0001$) & $\mathbf{4.200}$ & $\mathbf{7.32}$ & $3.360$ & $\mathbf{3.347}$ & $3.436$ & $-0.019\pm0.014$ & $-0.003\pm0.011$ \\
LGP, full, early stopping (15 runs) & $3.3949$ ($0.0154$) & $4.241$ & $7.38$ & $3.361$ & $3.390$ & $3.433$ & $-0.005\pm0.007$ & $+0.011\pm0.004$ \\
\quad mixture of the 15 runs & $3.3826$ & $4.219$ & $7.35$ & $3.360$ & $3.359$ & $3.429$ & $-0.017\pm0.008$ & $-0.001\pm0.004$ \\
LGP, full, no early stopping & $3.4128$ & $4.363$ & $7.55$ & $\mathbf{3.357}$ & $3.464$ & $\mathbf{3.417}$ & $+0.013\pm0.019$ & $+0.029\pm0.015$ \\
\bottomrule
\end{tabular}}
\end{center}
\end{table}

\begin{figure}[h]
\begin{center}
\includegraphics[width=\textwidth]{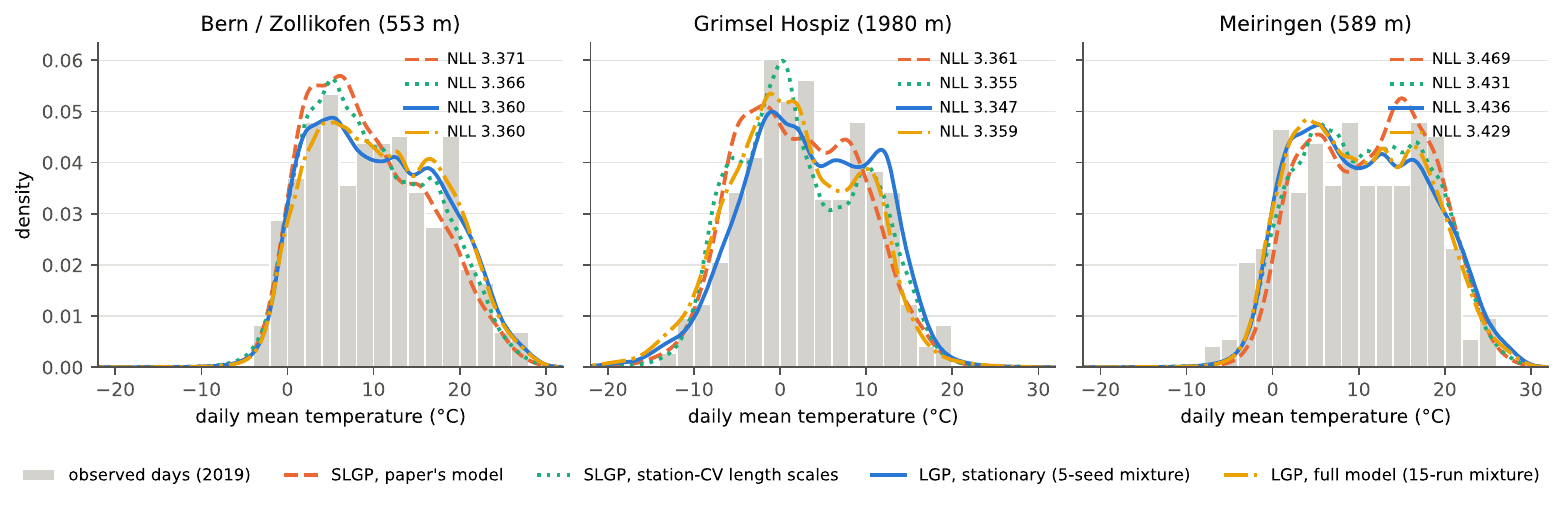}
\end{center}
\caption{Predictive densities at the three Bern test stations over a histogram of the $365$ observed daily means of 2019, with each
model's NLL at the station. Each model predicts one density per station, from its coordinates.}\label{fig:slgp-meteo}
\end{figure}

The stationary LGP has the lowest NLL, CRPS and RMSE, but the differences are not resolved. It is $0.019\pm0.014$ nats below the
paper's SLGP and level with SLGP at cross-validated length scales ($-0.003\pm0.011$), with ARD length scales learned by Fisher-identity
gradients instead of a grid search and its five seeds within $10^{-4}$ nats of each other. Refitted by NUTS with CRAN \texttt{SLGP}~2.0.0 and the package's settings ($2$ chains of $500$
warm-up and $500$ retained iterations) at either set of length scales, SLGP shows no divergent transitions or maximal trees, E-BFMI
of at least $0.89$, rank-normalised $\hat R$ of at most $1.035$ for every coefficient (bulk ESS at least $268$) and of at most $1.007$
for the per-draw mean test log-density (ESS at least $595$). In these refits the two chains' test NLLs differ by at most $0.0004$ nats, and the NUTS
predictive differs from the MAP of the same build by at most $0.0003$. The paper-$\ell$ rows of Table~\ref{tab:slgp-meteo}
are instead the saved fits of \citet{gautier2026continuous} from its own package build, in which NUTS is $0.0007\pm0.0010$ nats
behind the MAP. Neither difference between NUTS and MAP is resolved. The full
model gains nothing here. Its single runs are $0.011\pm0.004$ nats behind SLGP at cross-validated length scales, and they depend
on which three of the only $26$ distinct inputs are held out for stopping (the split holding out S\"antis, at $2{,}501$\,m one of
the three training stations above $1{,}900$\,m, gives $3.386$--$3.428$). Their mixture closes the gap ($-0.001\pm0.004$). Without early
stopping the encoder overfits the $26$ inputs (test NLL $3.385$ after $100$ steps, $3.413$ after $1{,}000$). It is then the best model at
the two stations below $600$\,m and the worst at Grimsel ($1{,}980$\,m). The median-gap scaling of the full model's preprocessing (Appendix~\ref{app:data}) shrinks the scale
of longitude here to $0$--$0.03$, against $0.07$--$0.12$ for latitude and $0.29$--$0.32$ for altitude, because the median temperature on
either side of the median longitude is equal or nearly so. Computationally this is SLGP's most favourable setting. With $26$ distinct
inputs the normaliser is evaluated once per station, and a MAP fit takes $1$--$2$\,s and NUTS $3$--$12$\,min. Our model takes
$10$--$13$\,min per fit on the GPU, of which $4$--$6$\,min are compilation and chain burn-in, followed by at most $1{,}000$ optimisation steps. Selecting the length scales costs SLGP about $33$\,CPU-minutes
for the paper's grid ($1{,}100$ MAP fits) and $114$\,CPU-minutes for the station cross-validation.

\paragraph{Bimodal sinusoid.}
The problem of Section~\ref{sec:sinus} has a known conditional density, so each fit can be compared with the truth, which is the
criterion of the consistency study of \citet{gautier2026continuous}. We draw it anew for four seeds of the generating process, each with its
own training and test rows (seed $42$ is the sample split by Table~\ref{tab:sinus}, and its split is that table's seed-$0$ split;
the other three samples are new, so the LGP rows here differ from the means of Table~\ref{tab:sinus}). On each we refit our stationary
model with the recipe of Section~\ref{sec:sinus}. SLGP uses the Gaussian kernel, as our model does, the exact basis scheme
(so that the model that is fitted is the model that is scored) and the response interval $[-2.5, 2.5]$. This interval contains the data
with probability $1-6\cdot10^{-7}$, and no test point falls outside it. Its search covers $252$ grid points (response length scales
$0.03$--$0.3$, covariate length scales $0.025$--$0.15$, $\sigma^2$ from $1$ to $10^4$ times the heuristic) on the training rows of
$N=400$ and $2{,}000$, at $1$--$2$ and $9$--$12$ hours of two-core fits per training set, and $N=10{,}000$ reuses the $N=2{,}000$ selection. Two selections
on the boundary were repeated on extended grids, on which they are interior. These are seed $45$ at $N=400$ on response length scales
up to $0.75$ and $\sigma^2$ up to $10^5$ times the heuristic, and seed $44$ at $N=2{,}000$ at $3$ and $10$ times its selected
$\sigma^2$ on the neighbourhood of its selected length scales. Besides NLL, CRPS and PIT, Table~\ref{tab:slgp-sinus} gives the excess NLL over the true density on the same test
rows, all of which lie inside the supports of both fitted models. It is a finite-sample difference and not an estimate of a
population Kullback--Leibler divergence from the Gaussian-mixture truth, which is infinite for both models, since both have
bounded support (Appendix~\ref{app:nonstat} for ours). The table also gives the binned squared Hellinger distance to the true density,
$H^2_{\mathrm{bin}} = 1-\sum_b\sqrt{P_bQ_b}$ over the bins of a shared grid of about $1{,}000$ response values plus the two tails
(the convention $\tfrac12\int(\sqrt p-\sqrt q)^2$), averaged over the test inputs. Binning can only lower it, and on every second or
fourth grid point it changes by at most $4\cdot10^{-4}$.

\begin{table}[h]
\caption{Bimodal sinusoid, four seeds of $2{,}000$ test rows each: mean over the seeds of the NLL, the excess NLL over the true
density on the same rows, the binned squared Hellinger distance $H^2_{\mathrm{bin}}$ to it, CRPS and PIT-KS; $\Delta$: paired NLL
difference to our model, mean $\pm$ standard deviation over the seeds (positive: worse than ours); time per fit, SLGP without its
search; seeds: the number of seeds averaged. Every SLGP fit uses the exact basis scheme; tuned: length scales and $\sigma^2$ selected
on held-out training rows, in a search space designed after test-scored probe fits on seed $42$ (see Protocol), so the
comparison is exploratory. $^\ast$Seeds $42$--$44$; on seed $45$ NUTS did not converge (see text).}
\label{tab:slgp-sinus}
\begin{center}
\footnotesize\setlength{\tabcolsep}{3pt}
\begin{tabular}{lcccccccc}
\toprule
model & NLL & excess & $H^2_{\mathrm{bin}}$ & CRPS & PIT-KS & $\Delta$NLL vs LGP & time & seeds \\
\midrule
\multicolumn{9}{l}{\emph{$N=400$}} \\
LGP, stationary & $-0.020$ & $0.233$ & $0.0621$ & $0.3608$ & $0.045$ & --- & $6$\,min & 4 \\
SLGP, package defaults, MAP & $+0.619$ & $0.873$ & $0.2903$ & $0.3771$ & $0.087$ & $+0.640\pm0.081$ & $8$\,s & 4 \\
SLGP, tuned, MAP & $-0.037$ & $0.217$ & $0.0485$ & $0.3658$ & $0.039$ & $-0.016\pm0.040$ & $40$\,s & 4 \\
SLGP, tuned, Laplace & $+0.632$ & $0.886$ & $0.3187$ & $0.6367$ & $0.144$ & $+0.652\pm0.195$ & $40$\,s & 4 \\
SLGP, tuned, NUTS$^\ast$ & $-0.021$ & $0.232$ & $0.0524$ & $0.3712$ & $0.043$ & $-0.008\pm0.041$ & $6.9$\,h & 3 \\
\multicolumn{9}{l}{\emph{$N=2{,}000$}} \\
LGP, stationary & $-0.185$ & $0.069$ & $0.0183$ & $0.3442$ & $0.027$ & --- & $8$\,min & 4 \\
SLGP, package defaults, MAP & $+0.340$ & $0.593$ & $0.1969$ & $0.3575$ & $0.066$ & $+0.525\pm0.115$ & $50$\,s & 4 \\
SLGP, tuned, MAP & $-0.177$ & $0.076$ & $0.0190$ & $0.3457$ & $0.021$ & $+0.008\pm0.009$ & $4$\,min & 4 \\
SLGP, tuned, Laplace & $+0.352$ & $0.606$ & $0.2376$ & $0.5595$ & $0.104$ & $+0.537\pm0.302$ & $4$\,min & 4 \\
\multicolumn{9}{l}{\emph{$N=10{,}000$}} \\
LGP, stationary & $-0.226$ & $0.027$ & $0.0065$ & $0.3388$ & $0.021$ & --- & $11$\,min & 4 \\
SLGP, package defaults, MAP & $+0.132$ & $0.386$ & $0.1260$ & $0.3489$ & $0.055$ & $+0.359\pm0.110$ & $5$\,min & 4 \\
SLGP, tuned, MAP & $-0.218$ & $0.035$ & $0.0091$ & $0.3391$ & $0.022$ & $+0.008\pm0.001$ & $30$\,min & 4 \\
SLGP, tuned, Laplace & $+0.075$ & $0.329$ & $0.1378$ & $0.4683$ & $0.072$ & $+0.301\pm0.194$ & $28$\,min & 4 \\
\bottomrule
\end{tabular}
\end{center}
\end{table}

With the package defaults SLGP scores far below the true density at every size (excess NLL $0.87$, $0.59$ and $0.39$ against $0.23$, $0.07$
and $0.03$ for our model). The default prior regularises these sharp densities too strongly, and the searches select $\sigma^2$ at
$100$ to $3\cdot10^4$ times the heuristic. Tuned, SLGP by MAP is level with our model at $N=400$ ($-0.016\pm0.040$ nats, and closer to
the truth in $H^2_{\mathrm{bin}}$, $0.049$ against $0.062$) and behind it at $2{,}000$ and $10{,}000$ ($+0.008\pm0.009$ and
$+0.008\pm0.001$ nats; $H^2_{\mathrm{bin}}$ $0.019$ and $0.009$ against $0.018$ and $0.0065$). Refining SLGP's response grid from
$1{,}401$ to $2{,}801$ and $5{,}601$ points changes the per-object log-density of any fit by at most $7\cdot10^{-4}$ and its test NLL by
at most $2\cdot10^{-4}$. NUTS with the package's settings ($2$ chains of $500$ warm-up and $500$ retained iterations) took
$5$--$8$\,h per fit at $N=400$. On three seeds it shows no divergent transitions, rank-normalised $\hat R$ of at most $1.023$ for every
coefficient (bulk ESS at least $299$) and of at most $1.004$ for the per-draw mean test log-density, although in two of them nearly every
tree reaches the maximal depth. There its predictive equals the MAP's within the resolution of three seeds (paired difference
$-0.001\pm0.005$ nats). On the fourth seed, with the weakest selected prior ($\sigma^2 = 2.5\cdot10^4$), it did not converge in these
iterations (coefficient $\hat R$ up to $1.14$, bulk ESS $12$), although its predictive is stable across the two chains ($\hat R$ $1.007$;
$-0.092$ against $-0.086$ for the MAP). We exclude it from the table. The package's Laplace approximation at the same hyperparameters is far worse ($+0.65$, $+0.54$ and $+0.30$
nats against our model), although it is implemented correctly. In all twelve Laplace fits the package's Hessian at the mode agrees with an
independent analytic Hessian of the negative log posterior to $3\cdot10^{-16}$ (itself confirmed by central differences of the
gradient), and its draws, whitened by it, are standard normal. The smallest eigenvalue of that Hessian equals the prior precision
$1/\sigma^2$ in every fit. Along directions that the data do not constrain at the mode, the Gaussian therefore keeps the prior
variance, which the selected $\sigma^2$ ($65$ to $2.5\cdot10^4$) makes large, whereas the likelihood does constrain those directions away
from the mode. The poor score comes from the Gaussian approximation at this weak prior and not from an implementation error. Our model learns its three
hyperparameters by Fisher-identity gradients in $4$--$13$\,min per fit on the GPU, while SLGP needs $40$\,s to $30$\,min per MAP fit after a
search of $1$--$12$ hours of two-core fits per training set.

\paragraph{Scaling on HAPPY.}
The third comparison uses the \textsc{Happy} A$\to$B benchmark of Table~\ref{tab:main} (ten inputs) on nested training subsets of
sample A, with $1{,}000$, $5{,}000$ and $10{,}000$ galaxies for both methods and $25{,}000$ and all $74{,}950$ for our model. Both are
scored on the same $10{,}000$ random galaxies of sample B (SLGP's exact scorer evaluates a $1{,}401$-point density for every distinct
test input). SLGP uses the exact basis scheme (the package's grid schemes cover every input and do not extend to ten),
$500$ basis functions and the response interval $[0, 1.62]$, which is $1.1$ times the largest redshift of sample A. It is run with the
package defaults (MAP and Laplace) and tuned on $1{,}000$ training galaxies (Mat\'ern-$3/2$ like ours, one length scale shared by the ten inputs). The search
that produced the reported fits covered $64$ grid points in $26$\,min of two-core fits on the first system. A repeat on the
second system over a wider $96$-point grid (response length scales up to $0.32$, $\sigma^2$ up to $1{,}000$ times the heuristic)
took $48$\,min and selected the same values, all interior (response length scale $0.08$). The selection is reused at $5{,}000$
and $10{,}000$, so that part of the curve is conditional on it. Our stationary model uses ARD Mat\'ern-$3/2$ kernels and no validation set; the full model is the \textsc{Happy}
configuration of Table~\ref{tab:config} with early stopping on $5\%$ of the training rows. Every fit is a single run.

\begin{figure}[h]
\begin{center}
\includegraphics[width=0.9\textwidth]{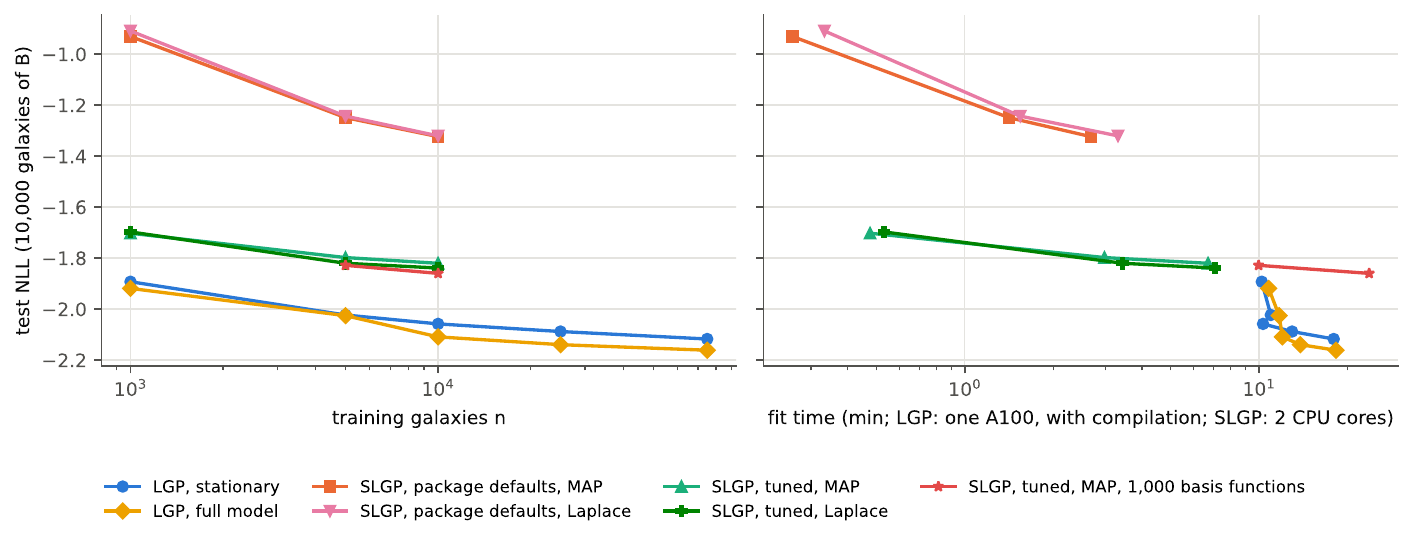}
\end{center}
\caption{\textsc{Happy} A$\to$B: test NLL on $10{,}000$ galaxies of sample B against the number of training galaxies (left) and
against the fit time (right; along each line $n$ increases; our model on one A100 including compilation, SLGP on two CPU cores,
excluding the $26$-minute search that selected its hyperparameters). ``Tuned'' SLGP uses the length scales and $\sigma^2$ selected on $1{,}000$ galaxies.}
\label{fig:slgp-happy}
\end{figure}

Our model has the lower NLL at every size (Figure~\ref{fig:slgp-happy}). At $10{,}000$ galaxies the full model reaches $-2.109$ and
the stationary one $-2.058$, against $-1.821$ for tuned SLGP by MAP, $-1.839$ by Laplace and $-1.32$ with the package defaults. These
are single runs, so their own variability is not measured. The gap of $0.2$--$0.3$ nats is large compared with the
seed-to-seed variation of the full-data fits of Table~\ref{tab:main}, although that variation does not quantify this curve's. SLGP improves by $0.12$ nats from $1{,}000$ to
$10{,}000$ galaxies and our full model by $0.19$, and it continues to $-2.161$ when trained on all of sample A ($-2.1557$ on all of sample
B, matching the four-seed mean of Table~\ref{tab:main}). The PIT statistic is $0.03$--$0.05$ for tuned SLGP by MAP and $0.009$--$0.037$ for the full model. Doubling SLGP's basis to $1{,}000$
functions at the same, not retuned, hyperparameters gains $0.03$--$0.04$ nats, so the random-feature representation of an
eleven-dimensional density may be one limit. The representation, ARD against shared length scales, the tuning budget and the inference
all differ between the two systems, and these runs do not separate their effects. The costs scale differently. SLGP's exact scheme
evaluates the basis at $101$ quadrature nodes for every training input, so a MAP fit takes $0.5$, $3$ and $7$\,min at $1{,}000$, $5{,}000$ and $10{,}000$ galaxies with $2.5$, $12$ and
$22$\,GB of memory ($24$\,min and $44$\,GB with $1{,}000$ basis functions), whereas ours takes $10$--$19$\,min per fit, nearly
independent of the training size up to $75{,}000$ galaxies, with $4$--$5$\,GB of GPU and about $4$\,GB of host memory at both
$10{,}000$ and $75{,}000$. NUTS for tuned SLGP at $1{,}000$ galaxies did not finish within the $2.5$\,h limit ($900$ of $1{,}000$
iterations). At $1{,}000$ galaxies, re-scoring the tuned MAP on $2{,}801$ and $5{,}601$ response points changes its test NLL by less
than $10^{-6}$ and the per-object log-density by at most $4\cdot10^{-4}$, and the package's Laplace Hessian agrees with an
independent analytic one to $4\cdot10^{-16}$, with its draws, whitened by it, standard normal.

\paragraph{Summary.}
On SLGP's own application our stationary model is level with SLGP at cross-validated length scales ($-0.003\pm0.011$ nats) and
ahead of the published configuration by $0.019\pm0.014$ (a difference the block bootstrap does not resolve decisively), and it learns
its length scales by gradient instead of on a grid. On the sinusoid, an exploratory comparison whose search space was informed by test-scored probe fits, SLGP needs its
variance added to the search to compete. Tuned, it is
level with our model at $400$ rows and $0.008$ nats behind at $2{,}000$ and $10{,}000$, whereas the package's Laplace approximation,
though implemented correctly, fails at these weak priors. On ten-dimensional photometric data our full model is $0.2$--$0.3$ nats ahead at
every training size, and SLGP's exact scheme grows to $22$--$44$\,GB at $10{,}000$ galaxies. These results hold for the stated budgets
and implementations (R and Stan on two CPU cores against JAX on one GPU). They do not establish limits of the SLGP model itself.

\end{document}